%% file: main.tex
\documentclass{article} 
\usepackage{iclr2025_conference,times}

\usepackage{wrapfig}
\usepackage{lipsum} 
\usepackage{wrapfig}

\input{math_commands.tex}

\usepackage{booktabs}      
\usepackage{multirow}      
\usepackage{amssymb}       

\usepackage{colortbl}
\usepackage{xcolor}

\usepackage{hyperref}
\usepackage{url}
\usepackage{multicol}
\usepackage{aliascnt}
\usepackage{adjustbox}
\usepackage{siunitx} 
\usepackage{multirow} 
\usepackage{enumitem}
\usepackage{booktabs} 
\usepackage{amssymb}  
\usepackage{amsmath}  
\usepackage{thm-restate}
\usepackage{graphicx} 
\usepackage{longtable} 
\usepackage{graphicx}
\usepackage{subcaption}
\usepackage{calc}
\usepackage[most]{tcolorbox}
\usepackage{listings}

\usepackage{bm}
\usepackage{bbm}
\usepackage{microtype}
\usepackage[capitalize,noabbrev]{cleveref}
\usepackage{makecell}
\usepackage{algorithmic}
\usepackage[ruled,vlined]{algorithm2e}
\usepackage{mathtools}
\usepackage{amsthm}
\usepackage{mathrsfs}
\theoremstyle{plain}
\theoremstyle{definition}

\theoremstyle{remark}

\usepackage{thm-restate}

\newtheorem{example}{Example}
\SetKwInput{KwInput}{Input}
\SetKwInput{KwOutput}{Output}
\usepackage{multicol}
\usepackage{pifont}
\usepackage{tcolorbox}
\usepackage{xcolor}
\usepackage{colortbl}
\usepackage{changepage}
\definecolor{myhighlight}{RGB}{220,240,255}

\tcbset{
  example/.style={
    colback=gray!10!white,
    colframe=gray!60!black,
    fonttitle=\bfseries,
    coltitle=white,
    boxrule=1pt,
    arc=4pt,
    left=6pt,
    right=6pt,
    top=6pt,
    bottom=6pt,
    title=Takeaway,
  }
}

\definecolor{uclablue}{rgb}{0.15, 0.45, 0.68}
\hypersetup{
    breaklinks,
    citecolor=uclablue,
    colorlinks=true,
}

\DeclareMathOperator{\clip}{clip}
\DeclareMathOperator{\TV}{TV}

\definecolor{linkColor}{rgb}{0.2,0.4,0.6}
\definecolor{myblue}{HTML}{0379AC}
\definecolor{myred}{HTML}{A50E50}
\definecolor{myorange}{RGB}{238, 133, 74}
\definecolor{latentcolor}{named}{cyan}
\definecolor{normalcolor}{RGB}{0, 0, 0}
\usepackage{marvosym}

\title{Listwise Policy Optimization: \\ Group-based RLVR as Target-Projection on the LLM Response Simplex}

\author{
Yun Qu$^{1,2}$\thanks{\ Work completed during an internship at Tencent.}, 
Qi Wang$^{1,\dagger}$, 
Yixiu Mao$^{1}$, 
Heming Zou$^{1,2,*}$, 
Yuhang Jiang$^1$, 
Yingyue Li$^1$, 
Wutong Xu$^1$, \\ 
Lizhou Cai$^1$, 
Weijie Liu$^{2}$,
Clive Bai$^{2}$, Kai Yang$^{2}$, Yangkun Chen$^{2}$, Saiyong Yang$^{2,\dagger}$, Xiangyang Ji$^{1,}$\thanks{\ Corresponding Authors}\\
\textbf{$^1$Department of Automation, Tsinghua University} \quad \textbf{$^2$LLM Department, Tencent}\\
\Letter~cheemswang@mail.tsinghua.edu.cn,\ stevesyang@tencent.com,\  xyji@tsinghua.edu.cn\\
}

\begin{document}
\maketitle
\let\oldthefootnote\thefootnote

\let\thefootnote\oldthefootnote

\input{sections/abstract}
\input{sections/introduction}

\input{sections/preliminary}
\input{sections/insight}

\input{sections/method}

\input{sections/experiments}

\section{Conclusion}
\label{sec:conclusion}

This work introduces a unified geometric framework for deep insight into group-based RLVR of LLMs. 
We show that existing policy gradient methods act as approximate target-projections on the response simplex and present LPO to perform this projection explicitly. 
LPO benefits from directly optimizing on the simplex, which improves optimization stability and yields monotonic performance improvements.
Moreover, the decoupled target-projection perspective opens up a flexible design space for developing rich and diverse optimization methods for RLVR of LLMs.

\textbf{Limitations and future work.}
Our current formulation primarily focuses on sequence-level projection within outcome reward settings.
Future research will explore step-level listwise projections and investigate broader divergences to fully unlock the potential of the decoupled framework.

\bibliography{iclr2025_conference}
\bibliographystyle{iclr2025_conference}

\appendix
\newpage

\input{sections/appendix}

\end{document}

%% file: math_commands.tex
\usepackage{amsmath,amsfonts,bm}
\usepackage{multicol}

\def\eqref#1{equation~\ref{#1}}

\def\1{\bm{1}}

\DeclareMathAlphabet{\mathsfit}{\encodingdefault}{\sfdefault}{m}{sl}
\SetMathAlphabet{\mathsfit}{bold}{\encodingdefault}{\sfdefault}{bx}{n}

\newcommand{\softmax}{\mathrm{softmax}}



%% file: sections/abstract.tex
\begin{abstract}

Reinforcement learning with verifiable rewards (RLVR) has become a standard approach for large language models (LLMs) post-training to incentivize reasoning capacity.
Among existing recipes, group-based policy gradient is prevalent, which samples a group of responses per prompt and updates the policy via group-relative advantage signals.
This work reveals that these optimization strategies share a common geometric structure: each implicitly defines a target distribution on the response simplex and projects toward it via first-order approximation.
Building on this insight, we propose Listwise Policy Optimization (LPO) to explicitly conduct the target-projection, which 
\textit{demystifies the implicit target} by restricting the proximal RL objective to the response simplex,
and then \textit{projects the policy} via exact divergence minimization.
This framework provides (i) monotonic improvement on the listwise objective with bounded, zero-sum, and self-correcting projection gradients, and (ii) flexibility in divergence selection with distinct structural properties through the decoupled projection step.
On diverse reasoning tasks and LLM backbones, LPO consistently improves training performance over typical policy gradient baselines under matched targets, while intrinsically preserving optimization stability and response diversity.
\end{abstract}

%% file: sections/introduction.tex
\section{Introduction}
\label{sec:intro}

\begin{wrapfigure}{r}{0.35\textwidth}
    \centering
    \vspace{-12pt}
    \includegraphics[width=0.35\textwidth]{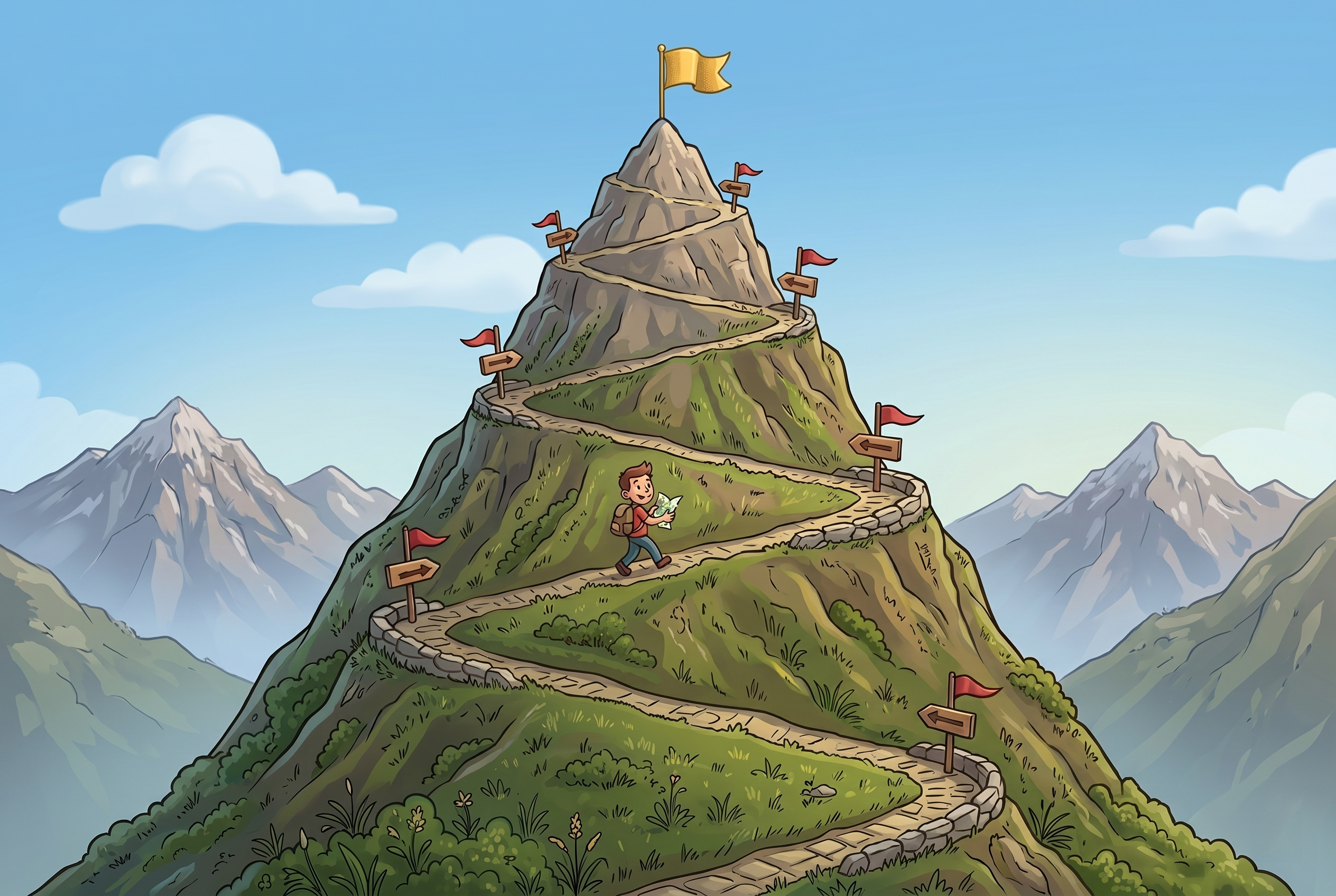}
    \caption{LPO iteratively ascends the reward landscape via explicit target-projection, enabling stable optimization and flexible divergence design.}
    \label{fig:lpo_overview}
    \vspace{-15pt}
\end{wrapfigure}
Recent advances have revealed the prominent potential of reinforcement learning with verifiable rewards~(RLVR) for large language models~(LLMs) post-training, which incentivizes reasoning capabilities on complex problem-solving tasks~\citep{guo2025deepseek,jaech2024openai,luo2025deepcoder}. 
In particular, critic-free, group-based RL paradigms, such as group relative policy optimization~(GRPO)~\citep{shao2024grpo}, have been widely adopted for RLVR. 
This setup samples a group of responses, scores them with a verifier, and performs policy gradient updates using group-relative advantages.
Further extensions in the literature ~\citep{liu2025drgrpo,yu2025dapo,tajwar2026maxrl,hu2025reinforce,chen2025cispo} have introduced critical refinements with special focus on advantage normalization and training stabilization.

\vspace{2pt}
\textbf{Group-based policy gradients as implicit target-projections.} 
Though these empirical refinements have proven effective, viewing them purely through the manner of advantage normalization obscures the intrinsic optimization mechanism.
By defining a listwise distribution~\citep{cao2007listnet,liu2025lipo} jointly over the sampled responses on a simplex,
this work provides a unified geometric perspective on group-based RL algorithms: their advantage formulas implicitly construct a reward-weighted softmax target distribution over the responses, with the target's sharpness configured by the normalization scheme.
Then, the standard policy gradient update acts merely as a first-order approximation of a reverse Kullback-Leibler~(KL)~\citep{kullback1951kullback} projection toward this implicit target. 
This integrated perspective not only elucidates the workings of current methods but also motivates the explicit design of the target-projection mechanism.

\vspace{2pt}
\textbf{From implicit approximation to explicit projection.} 
Explicit target projection has been studied in classical RL \citep{peters2010reps, abdolmaleki2018mpo, peng2019awr}.
However, the existence of continuous action spaces necessitates the use of function approximation.
In contrast, group-based RLVR exhibits a distinct and desirable property: the sampled responses for a prompt naturally form a finite simplex, allowing for the exact computation of both the target distribution and the projection in closed form. 
This makes it feasible to define clear separated goals between \textit{what distribution to target} and \textit{how to project toward it}, facilitating a seamless transition from implicit approximations to exact listwise optimization. 
Consequently, the central research question arises:

\textit{What properties emerge when this target-projection is made explicit, and how does this decoupled optimization space influence RLVR of LLMs?}

\vspace{2pt}
\textbf{Listwise Policy Optimization.}
In response to the above research question, this work develops Listwise Policy Optimization~(LPO) to enable explicit target-projection on the response simplex.
Specifically, LPO 
(i) explicates the implicit target by constraining the proximal RL objective to the sampled responses, yielding a closed-form solution with a controllable temperature,
and (ii) optimizes the policy by projecting it onto the target via divergence minimization on the response simplex.
The exact projection onto the simplex results in gradients that are bounded, zero-sum, and self-correcting by design, which induces variance reduction and stable optimization. 
Furthermore, the decoupled structure allows for flexible projection divergences, and we implement forward and reverse KL divergence as two representative instantiations. 
The resulting iterative target-projection algorithm provides provable monotonic improvement of the listwise reward per iteration.

\vspace{2pt}
\textbf{Contributions.}
This work aims to offer deeper insights into policy optimization in RLVR, focusing on understanding and identifying potential improvements.
The main contribution is two-fold:
\begin{enumerate}
    \item We provide a unifying analytical perspective, revealing that group-based policy gradient methods implicitly perform approximate target-projections on the response simplex.
\vspace{2pt}
    \item We develop LPO, an explicit target-projection framework that decouples listwise target construction from divergence projection, supported by theoretical analysis that proves improvement guarantee and characterizes projections' structural properties.
\end{enumerate}

Extensive evaluations across logic, mathematics, programming, and multi-modal reasoning tasks with diverse LLM backbones demonstrate the effectiveness of LPO: 
(i) LPO achieves higher expected Pass@1 and Pass@k accuracy during training compared to baselines under matched implicit target constructions;
(ii) decoupling the target from the projection accommodates diverse divergences, with a novel forward KL variant showing exceptional competitiveness; and 
(iii) LPO induces highly stable optimization trajectories while inherently preserving response diversity.

%% file: sections/preliminary.tex
\section{Preliminaries}
\label{sec:prelim}

\subsection{Reinforcement Learning with Verifiable Rewards}
\label{sec:problem}

RLVR has emerged as a critical post-training paradigm for incentivizing reasoning capabilities of LLMs~\citep{shao2024grpo, jaech2024openai}.
Let $x$ denote a prompt and $y = (y_1, \ldots, y_{|y|})$ a response of length $|y|$, generated autoregressively by a parameterized policy $\pi_\theta(y \vert x) = \prod_{i=1}^{|y|} \pi_\theta(y_i \vert x, y_{<i})$.
Given a reward function $R(x, y)$ and a reference policy $\pi_{\mathrm{ref}}$, the standard KL-regularized objective for RLVR~\citep{shao2024grpo} is defined as:
\begin{equation}
J_x(\pi_\theta) = \mathbb{E}_{y \sim \pi_\theta(\cdot \vert x)}[R(x,y)] - \beta\,D_{\mathrm{KL}}\bigl(\pi_\theta(\cdot \vert x) \,\Vert\, \pi_{\mathrm{ref}}(\cdot \vert x)\bigr),
\label{eq:V}
\end{equation}
where $\beta \ge 0$ controls the strength of the reference constraint. 
Following recent advances~\citep{yu2025dapo, qu2025mopps}, we primarily focus on rule-based outcome rewards, which are typically binary or sparse ($R \in [0, 1]$), without an explicit reference penalty, i.e., $\beta = 0$.

\subsection{Group-based Policy Gradient}
\label{sec:group_pg}

The dominant paradigm in RLVR is group-based policy gradient~(PG), represented by Group-Relative Policy Optimization~(GRPO)~\citep{shao2024grpo}.
For each prompt $x$, a behavior policy $\pi_b$, which is typically the pre-update snapshot $\pi_{\theta_{\mathrm{old}}}$, generates a group of $K$ responses $\{y_1, \ldots, y_K\}$, each assigned a reward $R_k$ forming the reward vector $R = [R_1, \ldots, R_K]^\top$.
These rewards are converted into group-relative advantages, forming the advantage vector $A = [A_1, \ldots, A_K]^\top$ via centering and scaling. For instance, GRPO uses $A_k = \frac{R_k - \mu_G}{\sigma_G}$, where $\mu_G$ and $\sigma_G$ are the group mean and standard deviation. Table~\ref{tab:implicit_temp} details other common normalization schemes.
The policy is typically updated by maximizing a clipped surrogate objective~\citep{schulman2017ppo,shao2024grpo}:
\begin{equation}
\mathcal{L}_{\mathrm{GRPO}}(\theta) = \mathbb{E}_{x, \{y_k\}_{k=1}^K\sim\pi_b(\cdot\vert x)}\left[\frac{1}{K}\sum_{k=1}^K \frac{1}{|y_k|}\sum_{i=1}^{|y_k|} \min\bigl(r_{k,i}\,A_k,\;\clip(r_{k,i},\,1{-}\epsilon,\,1{+}\epsilon)\,A_k\bigr)\right],
\label{eq:GRPO}
\end{equation}
where $r_{k,i}(\theta) = \frac{\pi_\theta(y_{k,i} \vert x, y_{k,<i})}{\pi_b(y_{k,i} \vert x, y_{k,<i})}$ is the importance ratio and $\epsilon$ is the clipping hyperparameter.

At the exact on-policy point ($\pi_\theta = \pi_{b}$), the importance ratios are identically one ($r_{k,i} = 1$). Consequently, for a fixed prompt $x$, the surrogate objective gradient reduces to the standard sequence-level group-based policy gradient~\citep{sutton1999policy}:
\begin{equation}
g_{\mathrm{PG}} = \frac{1}{K}\sum_{k=1}^K A_k\,\nabla_\theta \log\pi_\theta(y_k \vert x), \quad \text{where } \log\pi_\theta(y_k \vert x) \triangleq \frac{1}{|y_k|}\sum_{i=1}^{|y_k|} \log\pi_\theta(y_{k,i} \vert x, y_{k,<i}).
\label{eq:PG}
\end{equation}

%% file: sections/insight.tex
\section{Group-based Policy Gradient as Implicit Target-Projection}
\label{sec:analysis}

This section reinterprets group-based policy gradients through the lens of the listwise distribution.
We aim to explore: (i) the target distribution that these updates implicitly pursue, and (ii) the impact of different advantage normalization schemes on shaping that target.

\subsection{Listwise Distribution on the Response Simplex}

To formalize, we represent the policy's relative preference over the $K$ sampled responses for prompt $x$ as a listwise distribution $P_\theta$~\citep{cao2007listnet,rafailov2024dpo,liu2025lipo}:
\begin{equation}
P_{\theta,k} = \frac{\exp(s_{\theta,k})}{\sum_{j=1}^K \exp(s_{\theta,j})} = \softmax(s_\theta)_k, \quad \text{with}\ s_{\theta,k} = \log\frac{\pi_\theta(y_k \vert x)}{\pi_b(y_k \vert x)},
\label{eq:P}
\end{equation}
where $P_\theta$ reflects the extent to which $\pi_\theta$ prioritizes each response relative to $\pi_b$.
At the on-policy point ($\pi_\theta = \pi_b$), $P_\theta$ reduces to the uniform distribution $1/K$.
Since $P_{\theta,k} \ge 0$ and $\sum_k P_{\theta,k} = 1$, the vector $P_\theta$ lies on the
probability simplex $\Delta^{K-1} = \{p \in \mathbb{R}^K : p_k \ge 0,\, \sum_k p_k = 1\}$,
which we call the \textbf{\emph{response simplex}}.

\subsection{Group-based Policy Gradient as Approximate Reverse KL}
\label{sec:pg_revkl}
With the listwise distribution, we now reveal the underlying geometric property: standard group-based policy gradients implicitly perform target-projection via reverse Kullback-Leibler~(KL)~\citep{kullback1951kullback} minimization.

\begin{restatable}[Group-based policy gradient as reverse KL at on-policy]{proposition}{propPgRevkl}
\label{prop:pg_revkl}
Let $A \in \mathbb{R}^K$ be a zero-mean advantage vector, i.e., $\sum_{k=1}^K A_k = 0$, and let $w^{\ast} = \softmax(A)$. At the on-policy point ($\pi_\theta=\pi_b$), the policy gradient in Eq.~\ref{eq:PG} equals the negative gradient of the reverse KL divergence $D_{\mathrm{KL}}$:
\begin{equation}
g_{\mathrm{PG}} = \frac{1}{K}\sum_{k=1}^K A_k\,\nabla_\theta\log\pi_\theta(y_k \vert x) = -\nabla_\theta\,D_{\mathrm{KL}}(P_\theta \Vert w^{\ast})\Big|_{\pi_\theta=\pi_b}.
\end{equation}
\end{restatable}

This observation identifies $w^{\ast}=\softmax(A)$ as the \emph{implicit target} on the response simplex induced by the advantage design.
This equivalence is exact at the on-policy point, but the approximation error grows as the policy drifts from the sampling distribution. Concretely, the per-response coefficient discrepancy scales as $\mathcal{O}(\bar\delta \cdot (1+\|A\|_\infty) / K)$, where $\bar\delta = \max_k |\frac{\pi_\theta(y_k|x)}{\pi_b(y_k|x)} - 1|$ measures the degree of off-policy drift.
See Appendix~\ref{app:proof_pg_revkl} for detailed proof.

\subsection{Implicit Targets of Existing Methods}
\label{sec:adv_targets}

\begin{table}[t]
\centering
\caption{Advantages and implicit targets of existing group-based policy gradient methods.}
\vspace{-3pt}
\label{tab:implicit_temp}
\resizebox{\linewidth}{!}{
\begin{tabular}{lccc}
\toprule
Algorithm & Advantage $A_k$ & Implicit target $w^{\ast}$ & Temperature $\tau$ \\
\midrule
Dr.GRPO~\citep{liu2025drgrpo} / RLOO~\citep{ahmadian2024rloo} & $R_k - \mu_G$ & $\softmax(R)$ & $\sim1$ \\
GRPO~\citep{shao2024grpo} / DAPO~\citep{yu2025dapo} & $(R_k - \mu_G)/\sigma_G$ & $\softmax(R/\sigma_G)$ & $\sigma_G$ \\
MaxRL~\citep{tajwar2026maxrl} & $(R_k - \mu_G)/\mu_G$ & $\softmax(R/\mu_G)$ & $\mu_G$ \\
\bottomrule
\end{tabular}
}
\vspace{-10pt}
\end{table}

Table~\ref{tab:implicit_temp} summarizes the specific implicit targets induced by existing group-based PG algorithms.
Advantages in these methods take the form $A_k = (R_k - \mu)/\tau$ for various choices of centering $\mu$ and scaling $\tau$.
By the shift-invariance of softmax, the centering cancels and the target $w^{\ast}$ reduces to $\softmax(R/\tau)$, where $\tau$ acts as a temperature.
Different normalization schemes thus preserve the same reward ordering with the main difference in target sharpness, as detailed in Appendix~\ref{app:dapo}.

\textbf{From approximation to exact projection.}
This unifying view also suggests a natural refinement.
Since both the target $w^{\ast}$ and the listwise distribution $P_\theta$ lie on the finite response simplex, the projection can be performed in an exact manner.
Moreover, it provides a new lens on algorithm design worth investigating: exact projection allows for any statistical divergence, e.g., Forward KL, that were inaccessible under the current policy gradient paradigm.
Accordingly, the next section will develop a generalized framework.

%% file: sections/method.tex
\section{Listwise Policy Optimization}
\label{sec:lpo}
We now replace implicit policy gradient approximations with an explicit target-projection framework on the response simplex.
This framework decouples each iteration into two entangled steps:
\begin{equation}
\underbrace{w^{\ast} = \arg\max_{w \in \Delta^{K-1}}\; \hat{J}(w)}_{\text{(i) Target: \emph{what} distribution to aim for}}
\qquad\qquad
\underbrace{\theta' = \arg\min_{\theta}\; D\!\left(w^{\ast}\,\Vert\, P_\theta\right)}_{\text{(ii) Projection: \emph{how} to optimize toward it}}
\label{eq:tp}
\end{equation}
where $\hat{J}$ is a proximal objective on the simplex and $D$ is a divergence measure.
Next, we will detail the optimization steps, their implementation, and the theoretical analysis.

\subsection{Target Induced on the Response Simplex}
\label{sec:proximal}

To demystify the principled origin of the implicit target in group-based policy gradients,
we define a local proximal RL objective per prompt on the response simplex, which maximizes the expected reward subject to a trust region around the policy~\citep{schulman2017trpo}:
\begin{equation}
\max_{w \in \Delta^{K-1}} \hat{J}(w) = \sum_{k=1}^K w_k R_k 
- \tau\,D_{\mathrm{KL}}(w \Vert P_t),
\label{eq:J}
\end{equation}
where $P_t = \softmax(s_t)$ is the listwise distribution induced by the pre-update policy $\pi_t$, 
with $s_{t,k} = \log\bigl(\pi_t(y_k \vert x)/\pi_b(y_k \vert x)\bigr)$.
Equivalently, $P_t$ is $P_\theta$ from Eq.~\ref{eq:P} evaluated at $\theta = \theta_t$. 
Both $P_t$ and $s_t$ are held fixed while $\theta$ is updated.

\begin{restatable}[Listwise Gibbs target]{theorem}{thmGibbs}
\label{thm:gibbs}
The objective $\hat{J}(w)$ in Eq. (\ref{eq:J}) has a unique maximizer $w^\ast$:
\begin{equation}
w_k^{\ast} = \softmax(\phi)_k, \quad\text{with}\ \ \phi_k = \frac{R_k}{\tau} + s_{t,k}.
\label{eq:W}
\end{equation}
\end{restatable}
Theorem \ref{thm:gibbs} indicates that the target $w^{\ast}$ re-weights the baseline $P_t$ toward high-reward responses, with $\tau$ controlling the sharpness: $w^{\ast} \to \arg\max_k R_k$ as $\tau \to 0$, and $w^{\ast} \to P_t$ as $\tau \to \infty$.
Under the on-policy setup ($\pi_t = \pi_b$), $P_t$ degenerates to a uniform distribution and $w^{\ast} = \softmax(R/\tau)$ \textit{recovers} the implicit targets of existing methods (Proposition~\ref{prop:pg_revkl}), 
with $\tau$ now an explicit design parameter with trust-region interpretation rather than a byproduct of advantage normalization.
Notably, concurrent works~\citep{kaddour2026tpo, shu2026reference} explore a similar paradigm by explicitly formulating the target as a reward-based Gibbs distribution with varied reference policies.

As $K \to \infty$, the empirical response simplex approximates the full policy space, and Eq.~\ref{eq:J} recovers the KL-regularized RL objective 
$\max_{w} \mathbb{E}_{w}[R] - \tau D_{\mathrm{KL}}(w \Vert \pi_t)$~\citep{ziebart2010maxent,levine2018rl}, 
whose solution is $w^{\ast} \propto \pi_t(y)\exp(R(y)/\tau)$ with an intractable partition function.
Operating on a finite response simplex yields a tractable formulation and makes the implicit target explicit.

\textbf{Monotonic improvement guarantee.}
\label{sec:monotonic}
The proximal objective $\hat{J}(w)$ serves as a surrogate to the listwise reward $\hat{R}(w) = \sum_k w_k R_k$, 
establishing a performance improvement bound:

\begin{restatable}[Performance improvement bound]{theorem}{thmMonotonic}
\label{thm:monotonic}
Assume $|R_k| \le R_{\max}$. If the projection step achieves $\mathrm{TV}(P_{t+1}, w^{\ast}) \le \epsilon_{\mathrm{proj}}$, then
\begin{equation}
\hat{R}(P_{t+1}) \ge \hat{R}(P_t) + \underbrace{\tau\bigl[D_{\mathrm{KL}}(w^{\ast} \Vert P_t) + D_{\mathrm{KL}}(P_t \Vert w^{\ast})\bigr]}_{\text{target gain}\;\ge\;0} - \underbrace{2R_{\max}\epsilon_{\mathrm{proj}}}_{\text{projection error}}.
\end{equation}
\end{restatable}
The target gain in Theorem \ref{thm:monotonic} is the Jeffreys divergence~\citep{jeffreys1946}.
With perfect projection, i.e., $\epsilon_{\mathrm{proj}} = 0$, the reward strictly improves whenever $P_t \ne w^{\ast}$.
In the idealized full policy space, iterating the exact proximal update converges to the reward-maximizing policy, providing a limiting reference for the target-projection framework.
See Appendix~\ref{app:proof_monotonic} and Appendix~\ref{app:proof_optimality} for proofs.
\begin{restatable}[Idealized full-space convergence]{proposition}{propOptimality}
\label{prop:optimality}
Let $\pi_0(y) > 0$ for all $y$, and assume $R(y)$ is bounded. 
Under exact proximal updates 
$\pi_{t+1}(y) \propto \pi_t(y)\exp(R(y)/\tau)$, 
the iteration satisfies 
$\pi_t(y) \propto \pi_0(y)\exp(tR(y)/\tau)$ and 
$\mathbb{E}_{\pi_t}[R] \to \max_y R(y)$ as $t \to \infty$.
\end{restatable}

\begin{figure}
    \centering
    \includegraphics[width=0.95\linewidth]{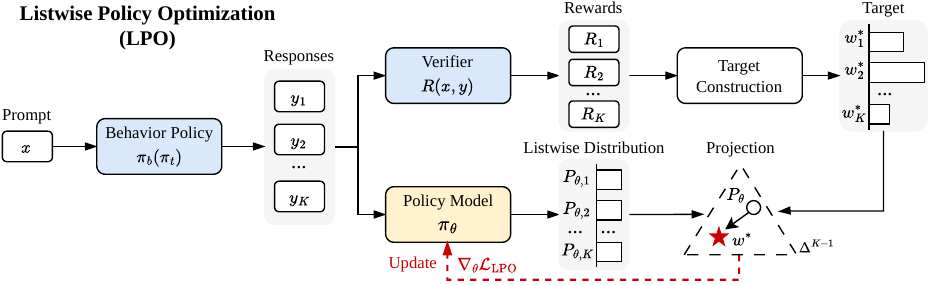}
    \vspace{-5pt}
    \caption{Illustration of LPO, which performs explicit target projection on the LLM response simplex, in contrast to the implicit approximations of group-based policy gradient methods.}
    \vspace{-5pt}
    \label{fig:lpo}
\end{figure}

\subsection{Projection for Policy Optimization}
\label{sec:projection}

With both the target $w^{\ast}$ in Eq.~\ref{eq:W} and the listwise distribution $P_\theta$ in Eq.~\ref{eq:P}
on $\Delta^{K-1}$, 
policy optimization reduces to a projection under a chosen divergence.
As representative choices, we develop the forward and reverse KL versions, with full derivations in Appendix~\ref{app:kl_gradients}.

\begin{example}[\textbf{Forward KL}]
    Minimizing the forward KL divergence $D_{\mathrm{KL}}(w^{\ast} \Vert P_\theta)$ gives:
    \vspace{-3pt}
\begin{equation}
\min \mathcal{L}_{\mathrm{LPO_{fwd}}}=D_{\mathrm{KL}}(w^{\ast} \Vert P_\theta)\Rightarrow \nabla_\theta\,\mathcal{L}_{\mathrm{LPO_{fwd}}} = \sum_{k=1}^K \underbrace{\bigl(P_{\theta,k} - w_k^{\ast}\bigr)}_{c_k^{\mathrm{fwd}}}\,\nabla_\theta\log\pi_\theta(y_k \vert x).
\label{eq:G}
\end{equation}
\vspace{-10pt}
\end{example}

The coefficient $c_k^{\mathrm{fwd}}$ measures the probability gap between the current policy and the target.
Although similar projection objectives exist in prior methods~\citep{abdolmaleki2018mpo,peng2019awr,shu2026reference}, they are implemented in a pointwise manner, treating each response independently without relative comparison.
In contrast, LPO performs projection on the response simplex via shared normalization, which couples across responses.
Furthermore, this yields the following desirable properties:

\begin{restatable}[Gradient coefficient properties]{corollary}{corBounded}
\label{cor:bounded}
The forward KL gradient coefficients $c_k^{\mathrm{fwd}}$ satisfy: (a) bounded: $|c_k^{\mathrm{fwd}}| \le 1$; (b) zero-sum: $\sum_k c_k^{\mathrm{fwd}} = 0$; (c) self-correcting: $c_k^{\mathrm{fwd}} \to 0$ as $P_\theta \to w^{\ast}$.
\end{restatable}

\begin{restatable}[Mode-Coverage]{corollary}{corModeCoverage}
\label{cor:mode_coverage}
If $w_k^{\ast} \ge \alpha$ and $D_{\mathrm{KL}}(w^{\ast} \Vert P_\theta) \le D$, then $P_{\theta,k} > \alpha\exp\left(-D/\alpha-1\right)$.
\end{restatable}

The zero-sum property acts as a built-in control variate for variance reduction~\citep{sutton1988learning}.
The bounded and self-correcting properties further improve optimization stability.
Moreover, Corollary~\ref{cor:mode_coverage} provides a log-barrier against mode collapse, ensuring response diversity.
Recently, \citet{kaddour2026tpo} adopts a very similar listwise forward KL projection, empirically corroborating its practical efficacy.

\begin{example}[\textbf{Reverse KL}]
    Minimizing the reverse KL divergence $D_{\mathrm{KL}}(P_\theta \Vert w^{\ast})$, with logit gap $d_k = s_{\theta,k} - \phi_k$ (the difference between the current policy and the target) and its $P_\theta$-weighted mean $\bar{d} = \sum_j P_{\theta,j}\, d_j$,  yields the following gradient:
    \vspace{-8pt}
\begin{equation}
\min \mathcal{L}_{\mathrm{LPO_{rev}}}=D_{\mathrm{KL}}(P_\theta \Vert w^{\ast})\Rightarrow \nabla_\theta\,\mathcal{L}_{\mathrm{LPO_{rev}}} = \sum_{k=1}^K \underbrace{P_{\theta,k}\bigl(d_k - \bar{d}\bigr)}_{c_k^{\mathrm{rev}}}\,\nabla_\theta\log\pi_\theta(y_k \vert x).
\label{eq:Rprime}
\end{equation}
\vspace{-10pt}
\end{example}

Similar to the forward KL, the gradient coefficient $c_k^{\mathrm{rev}}$ is zero-sum and self-correcting.
Minimizing reverse KL is equivalent to maximizing the proximal objective $\hat{J}$ (See Proposition~\ref{prop:proximal_revkl}), and it decomposes as $H(P_\theta) + \sum_k P_{\theta,k}\, \phi_k$: revealing an implicit entropy bonus (Appendix~\ref{app:entropy}).
At the on-policy point, the gradient of this objective exactly recovers the standard policy gradient (Proposition~\ref{prop:pg_revkl}).

\begin{algorithm}[t]
\caption{Listwise Policy Optimization (LPO)}
\label{alg:lpo}
\begin{algorithmic}[1]
\REQUIRE Policy parameters $\theta$, temperature $\tau > 0$, batch size $B$, inner epochs $E$, step size $\eta$
\FOR{each training iteration}
  \STATE Set behavior policy $\pi_b \leftarrow \pi_\theta$, pre-update policy $\pi_t \leftarrow \pi_\theta$
  \STATE Sample a batch of prompts $\mathcal{B} = \{x_i\}_{i=1}^B$
  \STATE For each $x \in \mathcal{B}$, generate responses $\{y_{k}\}_{k=1}^K \sim \pi_b(\cdot \vert x)$ and compute rewards $\{R_k\}_{k=1}^K$ 
  \STATE \textbf{Compute Target:} $w^{\ast}(x) = \softmax(\phi(x))$ via Eq.~\ref{eq:W} for all $x \in \mathcal{B}$
  \FOR{epoch $e = 1$ \TO $E$}
    \STATE \textbf{Compute Coefficients:} $c_k(x)$ via Eq.~\ref{eq:G} (forward KL) or Eq.~\ref{eq:Rprime} (reverse KL)
    \STATE \textbf{Gradient Update:} $\theta \leftarrow \theta - \eta \frac{1}{B} \sum_{x \in \mathcal{B}} \sum_{k=1}^K c_k(x) \nabla_\theta \log \pi_\theta(y_{k} \vert x)$
  \ENDFOR
\ENDFOR
\end{algorithmic}
\end{algorithm}

\subsection{Practical Implementation}
\label{sec:practical}

The overall LPO procedure is summarized in Algorithm~\ref{alg:lpo}. 
The training pipeline is identical to standard group-based RL algorithms, with no additional computational cost.

\textbf{Temperature as an adaptive baseline.} 
While the temperature $\tau$ could theoretically be treated as a trust-region hyperparameter, we intentionally avoid introducing new tuning burdens.
Instead, we adapt $\tau$ using the group-relative advantage normalization statistics of existing methods, e.g., $\tau = \sigma_G$ for GRPO or $\tau = \mu_G$ for MaxRL.
This allows us to isolate gains from exact listwise projection while preserving the target temperature used by prior methods.

%% file: sections/experiments.tex
\section{Main Empirical Results}
\label{sec:experiments}

\subsection{Experimental Setup}
\label{sec:setup}

We evaluate LPO across four representative domains of reasoning: \textbf{logic, mathematics, programming, and multi-modal geometry}.
To assess generality, we benchmark across a \textbf{diverse set of LLM backbones} spanning different model sizes (1.5B–14B) and various LLM families.
We track the training performance by plotting the curves for expected Pass@1 (average accuracy over rollouts) and Pass@k~\citep{chen2021evaluating}, with the specific k configurations detailed per benchmark.

\textbf{Logical Reasoning.}
We adopt the Countdown Game, which requires composing given numbers using basic operations to match a target value.
We train on a subset of Countdown-34 dataset~\citep{tinyzero} and evaluate on both Countdown-34 and the harder Countdown-4.
We primarily use Qwen3-4B-Base~\citep{yang2025qwen3} and further evaluate models from other families in Sec.~\ref{exp:llm_families}.

\textbf{Mathematical Reasoning.}
We train on the MATH dataset~\citep{hendrycksmath2021} using Qwen3-1.7B-Base and Qwen3-8B-Base~\citep{yang2025qwen3}.
Evaluation is conducted on standard benchmarks following \citet{qu2025mopps,gao2025prompt}: AIME24, AIME25, AMC23, MATH500~\citep{lightman2023let}, Minerva Math~\citep{lewkowycz2022solving}, and OlympiadBench~\citep{he2024olympiadbench}.
In Appendix~\ref{appsec:polaris}, we scale to Qwen3-14B-Base on the larger Polaris dataset~\citep{Polaris2025}.

\textbf{Programming.}
We train and evaluate Qwen3-1.7B-Base on the respective training and test splits of the PRIME code dataset~\citep{cui2025prime}.

\textbf{Multi-Modal Geometry.}
Geometry problems require multi-modal understanding and reasoning. We train Qwen2.5-VL-3B-Instruct~\citep{bai2025qwen2} on the training split of the Geometry3k dataset~\citep{lu2021inter, geometry3k_dataset} and evaluate it on the test split.

\textbf{Baselines and LPO Variants.} 
We compare against three representative group-based policy gradient (\textbf{PG}) methods with varied target temperature designs: \textbf{GRPO ($\tau = \sigma_G$), Dr.GRPO ($\tau = 1$), and MaxRL ($\tau = \mu_G$)}.
To ensure a rigorous apples-to-apples comparison, we isolate the effect of the gradient formulation from temperature scaling by implementing LPO variants for each baseline.
Specifically, we develop forward ($\boldsymbol{\mathrm{LPO_{fwd}}}$) and reverse KL ($\boldsymbol{\mathrm{LPO_{rev}}}$) versions that use the exact same temperature $\tau$ as their corresponding PG counterpart.
The paired evaluation ensures that any performance differences are attributable to explicit listwise projection rather than temperature tuning.
We implement baselines and LPO with the verl framework~\citep{sheng2024verl}.

Additional implementation details are provided in Appendix~\ref{appsec:implementation}, together with extended experimental results in Appendix~\ref{appsec:results} and prompt examples in Appendix~\ref{appsec:dataexample}.

\begin{figure}
    \centering
    \includegraphics[width=\linewidth]{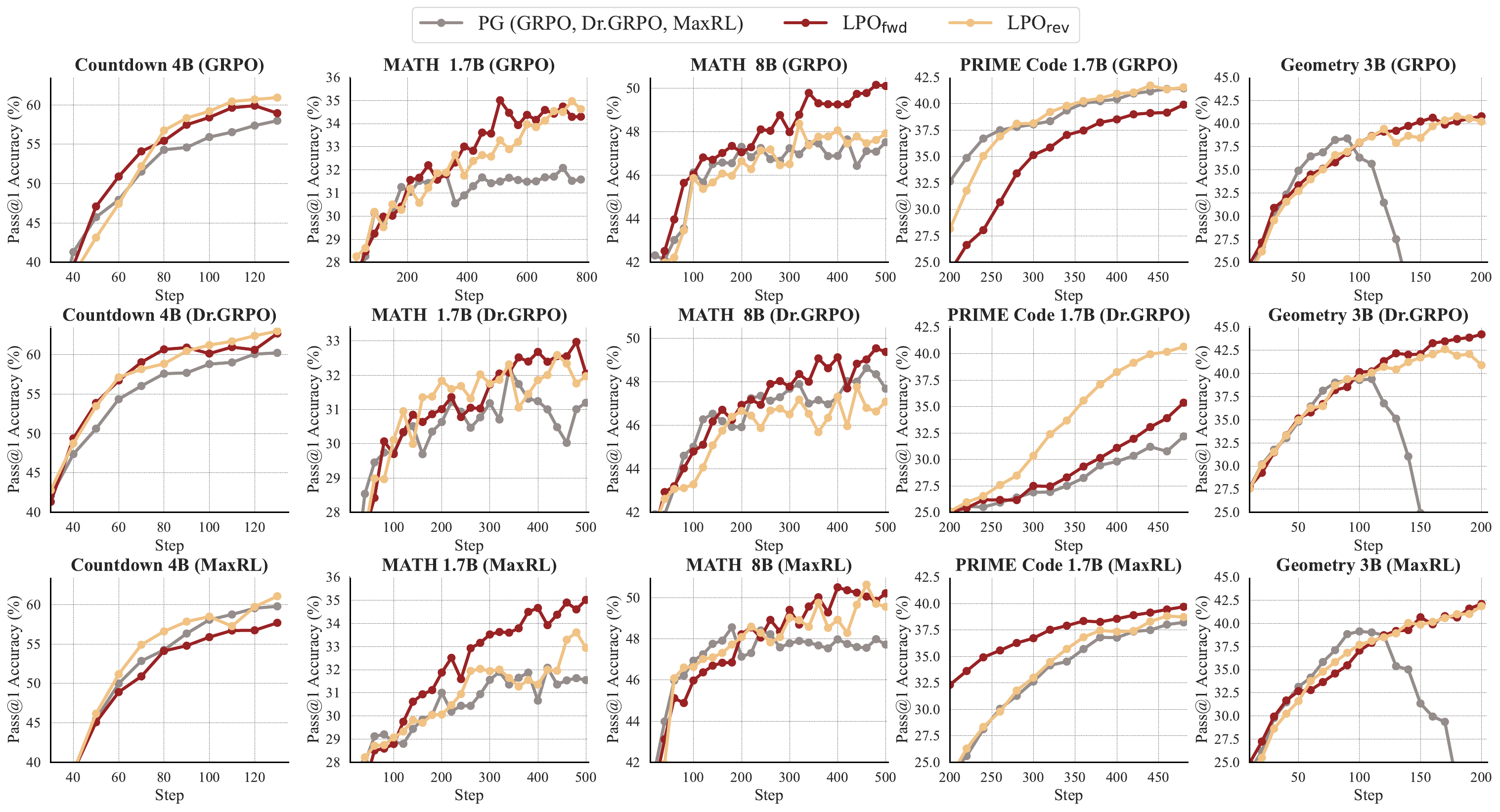}
    \vspace{-20pt}
    \caption{Training curves of Pass@1 accuracy. Two LPO variants ($\mathrm{LPO_{fwd}}$, $\mathrm{LPO_{rev}}$) are evaluated against group-based PG baselines (GRPO, Dr.GRPO, MaxRL, shown from top to bottom) across various LLM backbones and reasoning tasks with corresponding temperature designs.}
    \label{fig:pass1}
\end{figure}
\subsection{Training Performance}
\label{sec:training_performance}

\textbf{Performance gains.} 
Under paired temperature configurations, LPO consistently outperforms group-based PG baselines. For Pass@1 accuracy in Fig.~\ref{fig:pass1}, both LPO variants demonstrate efficient and improved training performance, exceeding their corresponding PG baselines in nearly all settings (13/15 for $\mathrm{LPO_{fwd}}$ and 13/15 for $\mathrm{LPO_{rev}}$). This advantage also extends to Pass@k evaluations in Fig.~\ref{fig:passk}, where both LPO variants continue to surpass the implicit PG methods (15/15 for $\mathrm{LPO_{fwd}}$ and 11/15 for $\mathrm{LPO_{rev}}$).
Together, these consistent gains suggest that replacing first-order advantage approximations with exact listwise projection on the response simplex offers a promising paradigm for improving the training efficiency and performance of RLVR.

\textbf{Projection divergence effects.}
Comparing the two variants reveals an empirical distinction: $\mathrm{LPO_{fwd}}$ outperforms $\mathrm{LPO_{rev}}$ in 13/15 scenarios for Pass@k.
This observation aligns well with the expectation: the mode-coverage property inherent to forward-KL actively preserves reasoning diversity for a broader distribution of valid solution paths.
More broadly, this highlights the flexibility of the decoupled target-projection framework, suggesting that exploring alternative projection divergences could unlock further unique optimization properties.

\textbf{Robustness across temperature parameterizations.} 
We observe that the optimal implicit temperature strategy $\tau$ is highly task-dependent, with no single design consistently dominating across all benchmarks. Despite this task-varying behavior, LPO delivers stable performance gains under all tested $\tau$ designs.
This indicates that exact listwise projection provides a robust optimization mechanism, yielding benefits that are largely orthogonal to the underlying temperature heuristic.

\begin{figure}
    \centering
    \includegraphics[width=\linewidth]{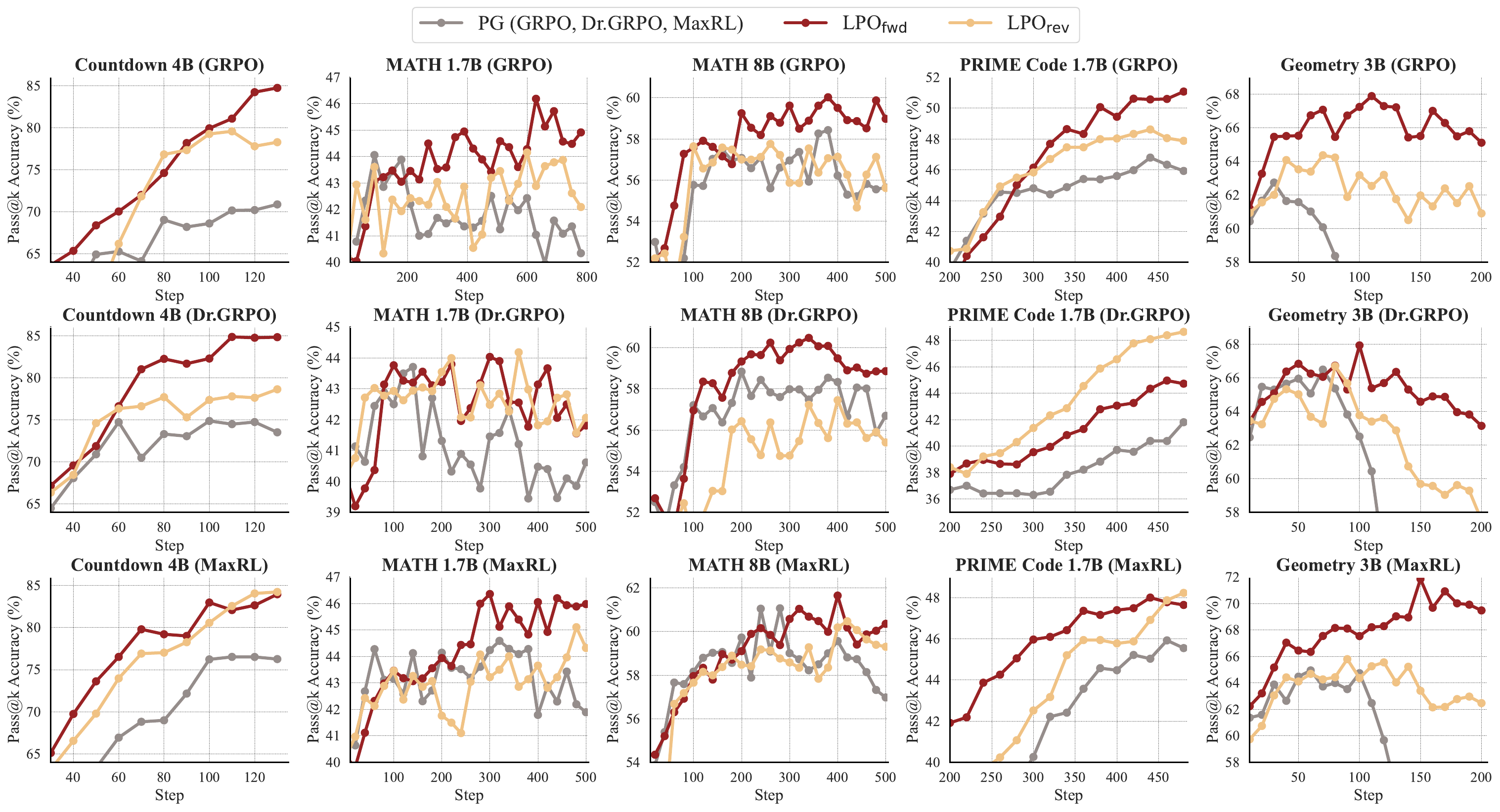}
    \vspace{-20pt}
    \caption{Pass@k training curves. LPO variants ($\mathrm{LPO_{fwd}}$, $\mathrm{LPO_{rev}}$) are evaluated against group-based PG baselines (GRPO, Dr.GRPO, MaxRL, shown from top to bottom) across various LLMs and tasks under paired temperature settings. Specific $k$ configurations are detailed per benchmark.}
    \label{fig:passk}
\end{figure}

\subsection{Training Dynamics}
\label{sec:training_dynamics}

To better understand the underlying optimization behaviors and validate our theoretical analysis, we track key training metrics: response entropy, gradient norm, and response length.

\textbf{Response entropy and exploration preservation.} 
As shown in Fig.~\ref{fig:grpodynamic} (top), both LPO variants generally maintain higher response entropy than PG baselines.
This corresponds to the projection properties:
$\mathrm{LPO_{rev}}$ corresponds to a maximum-entropy objective,
while $\mathrm{LPO_{fwd}}$ exhibits mode-covering behavior.
This sustained diversity directly explains the robust Pass@k improvements, positioning listwise projection as a principled remedy for the entropy collapse in RLVR.

\textbf{Gradient norms and optimization stability.} 
Fig.~\ref{fig:grpodynamic} (middle) reveals that LPO variants exhibit lower and more stable gradient norms compared to group-based PG methods.
This empirical stability is consistent with Corollary~\ref{cor:bounded}: LPO's exact projection on the response simplex yields controlled gradient coefficients, leading to stable optimization dynamics.

\begin{figure}
    \centering
    \includegraphics[width=\linewidth]{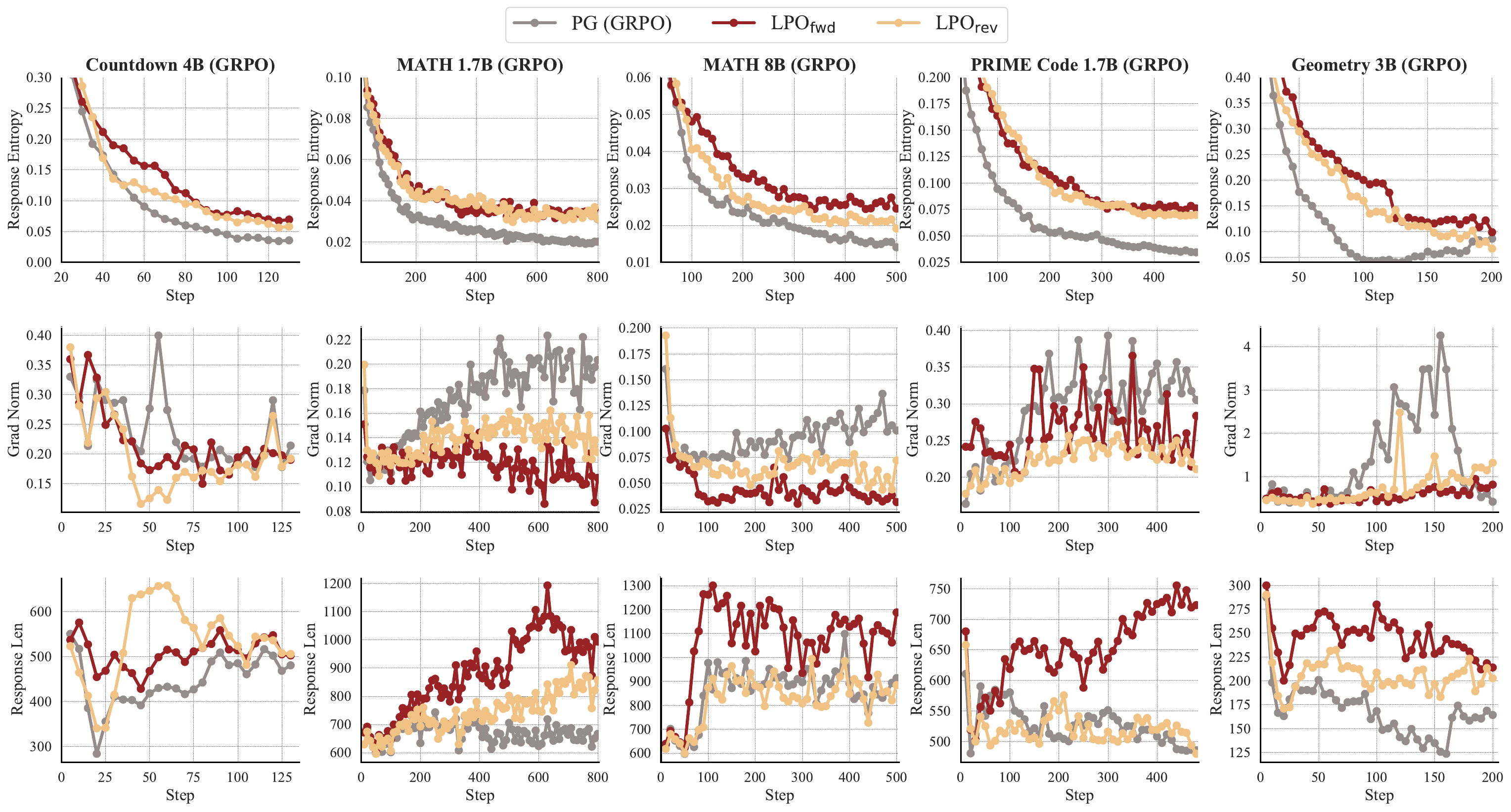}
    \vspace{-20pt}
    \caption{Training dynamics of LPO variants and GRPO. Rows from top to bottom respectively show the curves of response entropy, gradient norms, and response lengths.}
    \label{fig:grpodynamic}
\end{figure}

\textbf{Response length and reasoning behaviors.} 
Fig.~\ref{fig:grpodynamic} (bottom) shows that LPO tends to generate longer responses than PG.
As increased length often correlates with more detailed reasoning chains~\citep{yu2025dapo}, this is consistent with LPO encouraging more extensive exploration.
$\mathrm{LPO_{fwd}}$'s maximum length aligns with its mode-covering property, which promotes diverse reasoning paths.

\subsection{Additional Analysis}

\subsubsection{Listwise vs.\ Pointwise Projection}
\label{sec:vs_pointwise}

To highlight the contribution of the listwise projection, we ablate the listwise policy distribution in Eq.~\ref{eq:P} while keeping the target in Eq.~\ref{eq:W} unchanged.
This recovers the \textbf{pointwise projection} with forward KL~\citep{peters2010reps,abdolmaleki2018mpo,peng2019awr}, defined as $\mathcal{L}_{\mathrm{point}} = -\sum_{k} w_k^{\ast}\log\pi_\theta(y_k \vert x)$.
As shown in Fig.~\ref{fig:pointwise}, this pointwise variant suffers a severe performance drop.
This failure stems from the lack of a coupled competitive mechanism across responses in pointwise updates, resulting in unstable optimization.
Conversely, both group-based PG and LPO intrinsically employ a built-in control variate that stabilizes training.
These results suggest that our performance gains stem not merely from the target design, but from successfully marrying exact target fitting with the crucial structural variance reduction provided by the listwise projection.
Detailed properties of the two projections are deferred to Appendix~\ref{app:mpo}.

\begin{figure}[t]
\centering

\begin{minipage}[t]{0.49\textwidth}
\centering
\includegraphics[width=\linewidth]{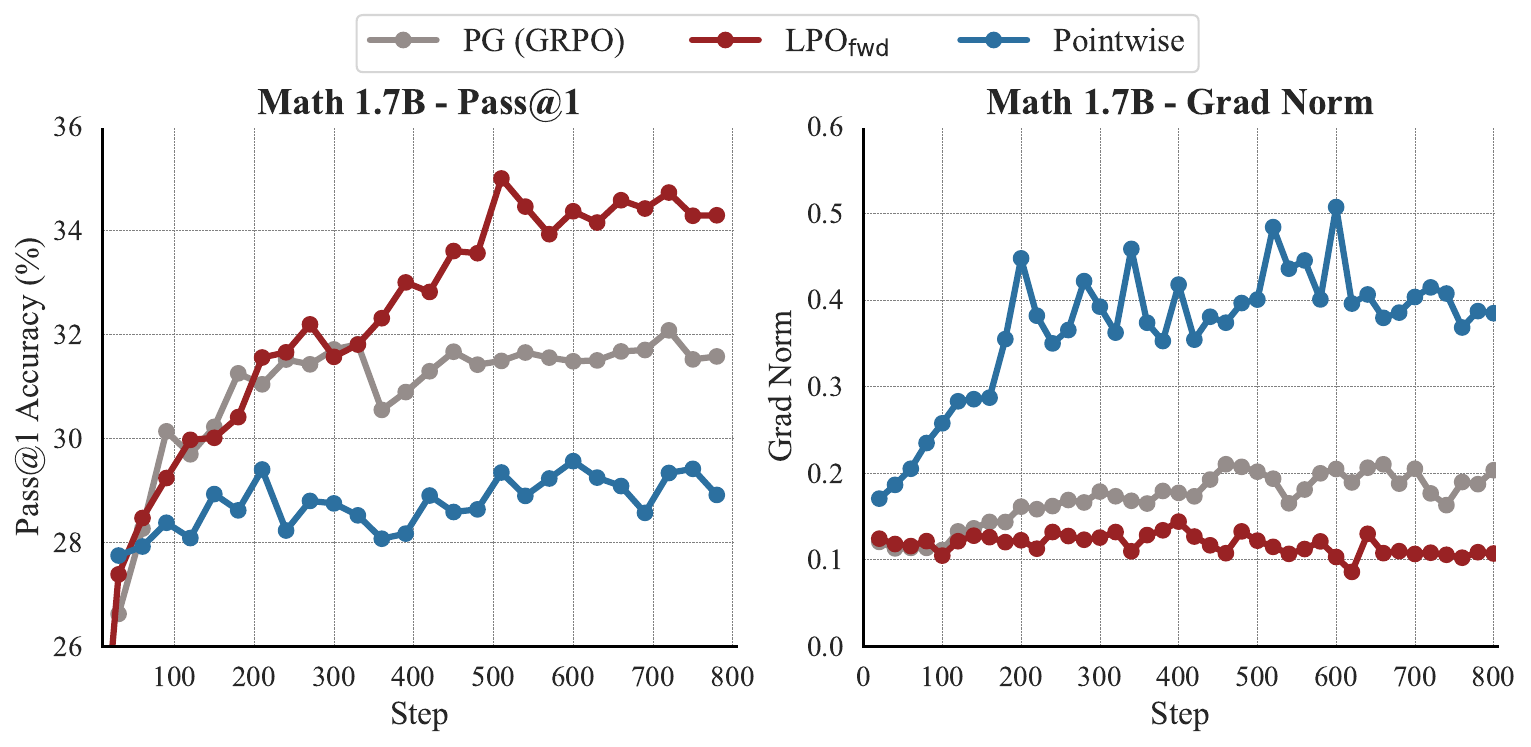}
\vspace{-20pt}
\caption{Ablation comparing listwise LPO with pointwise projection and GRPO baselines on MATH (Qwen3-1.7B-Base).}
\label{fig:pointwise}
\end{minipage}
\hfill
\begin{minipage}[t]{0.49\textwidth}
\centering
\includegraphics[width=\linewidth]{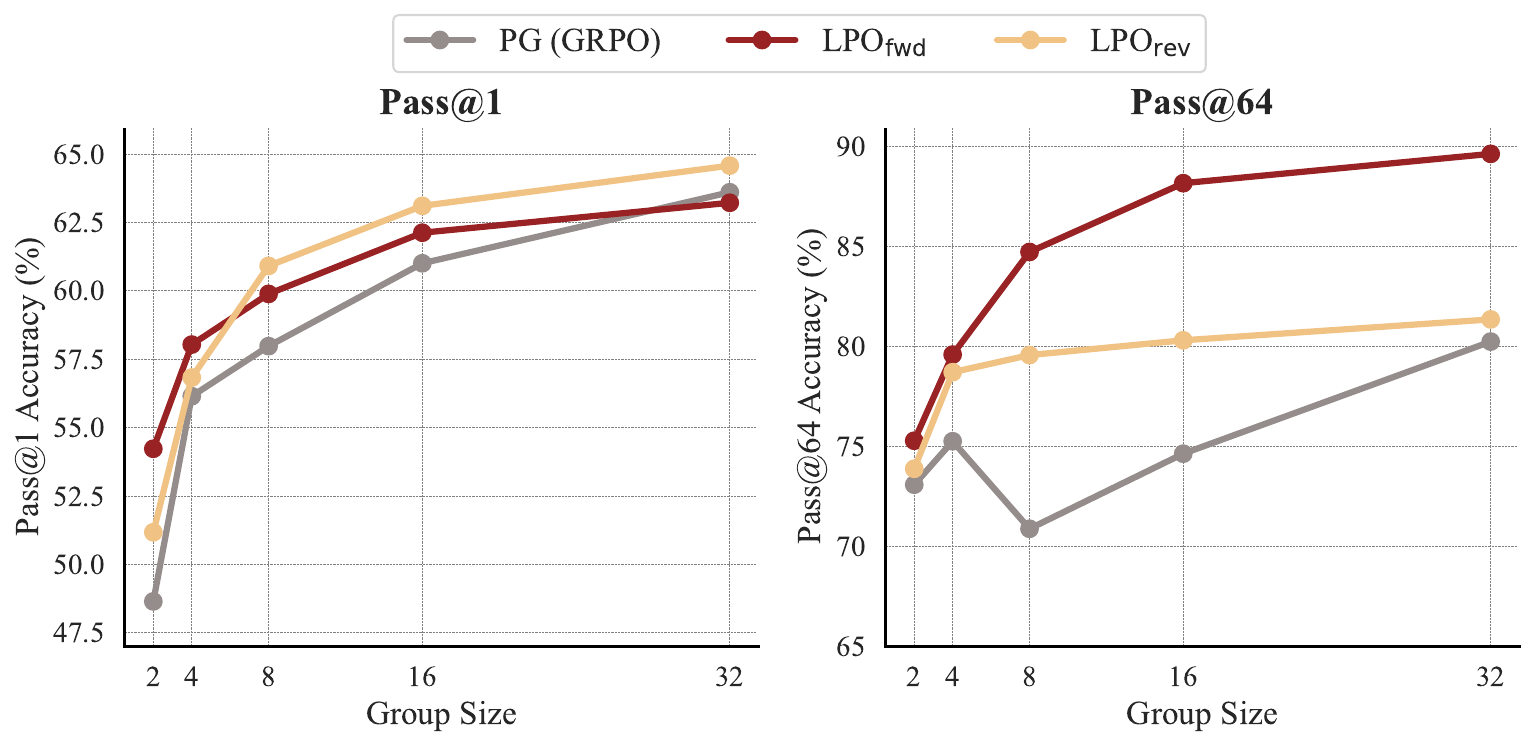}
\vspace{-20pt}
\caption{Effect of varying group sizes $K \in \{2, 4, 8, 16, 32\}$ on Countdown.}
\label{fig:groupsize}
\end{minipage}
\vspace{-5pt}
\end{figure}

\subsubsection{Effect of Group Size $K$}
\label{sec:group_size}

We investigate the impact of the sampled group size $K$ on Countdown.
As shown in Fig.~\ref{fig:groupsize}, across the tested group sizes ($K \in \{2, 4, 8, 16, 32\}$), both LPO variants remain highly competitive with GRPO, with advantages being particularly pronounced at smaller group sizes.
This suggests that explicit listwise projection improves sample efficiency, which stabilizes updates under limited samples.
Furthermore, LPO variants exhibit distinct scaling behaviors that validate their theoretical properties: $\mathrm{LPO_{rev}}$ achieves stronger Pass@1 performance, while $\mathrm{LPO_{fwd}}$ scales exceptionally well on Pass@64, supporting its mode-coverage property which structurally preserves reasoning diversity.

\subsubsection{Generalization across LLM Families}
\label{exp:llm_families}

To evaluate the generalizability of LPO, we conduct experiments across four prominent LLM families: Qwen, DeepSeek, Mistral, and Llama.
As illustrated in Fig.~\ref{fig:families} in Appendix~\ref{appsec:llmfamily}, LPO consistently delivers performance gains on the Countdown task, regardless of the underlying model architecture or training paradigm.
This consistent improvement across diverse backbones suggests that LPO is not sensitive to a specific model architecture, but rather benefits from the fundamental robustness of the listwise projection framework.
These results indicate that LPO can serve as a model-agnostic approach for improving reasoning performance in RLVR.

%% file: sections/appendix.tex
\section*{Appendix Overview}

This appendix provides supplementary proofs, conceptual discussions, and experimental details supporting the main text. It is organized as follows:

\begin{itemize}[leftmargin=10pt]
  \item \textbf{Appendix~\ref{sec:related} (Related Works):}  
  reviews literature on RLVR, RL as probabilistic inference, and listwise formulations.

  \item \textbf{Appendix~\ref{app:proofs_main} (Proofs):}  
  provides detailed mathematical derivations for all theoretical claims.

  \item \textbf{Appendix~\ref{app:additional} (Additional Discussions):}  
  expands on the framework's conceptual and practical scope. It unifies existing group-based RLVR algorithms, compares listwise projection with pointwise and preference optimization, explores future extensions.

  \item \textbf{Appendix~\ref{appsec:implementation} (Implementation Details):}  
  outlines the experimental setup, including tasks, LLM backbones, and training details.

  \item \textbf{Appendix~\ref{appsec:results} (Extended Experimental Results):}  
  reports further empirical findings, including scalability validation, on-policy optimization, extended training dynamics, and generalization across diverse LLM families.

  \item \textbf{Appendix~\ref{appsec:dataexample} (Data Examples):}  
  presents representative data examples used across the evaluated reasoning tasks.
\end{itemize}

\section{Related Works}
\label{sec:related}

\paragraph{Reinforcement learning with verifiable rewards.}
The alignment and reasoning capabilities of LLMs have been significantly advanced by RL~\citep{ouyang2022training, bai2022training}, initially dominated by PPO~\citep{schulman2017ppo} with a learned value model.
The emergence of RLVR for reasoning tasks~\citep{jaech2024openai, guo2025deepseek} has driven a paradigm shift toward critic-free, group-based policy gradient methods~\citep{shao2024grpo, ahmadian2024rloo, li2023remax}, which sample multiple responses per prompt and derive advantages entirely from within-group reward statistics.
A subsequent line of work has refined this paradigm, introducing novel advantage normalization~\citep{liu2025drgrpo, hu2025reinforce, tajwar2026maxrl}, trust-region mechanics~\citep{yu2025dapo, chen2025cispo}, and sequence-level scaling~\citep{zheng2025gspo}.
Recent theoretical works have sought to uncover the underlying mechanics of these methods~\citep{mroueh2025grpo,vojnovic2025alignment}.
Our LPO framework provides a unifying perspective by revealing that major group-based methods share the same target-projection geometry.
Concurrently, FlowRL~\citep{zhu2025flowrl} minimizes reverse KL against a Gibbs target approximated by a learned partition function network,
while contemporaneous TPO~\citep{kaddour2026tpo} similarly adopts cross-entropy on tilted simplex targets.
Differently, LPO contributes a unifying analytical Target-Projection framework that recovers existing group-based methods, and admits multiple divergences with provably satisfying properties.

\paragraph{RL as probabilistic inference.}
The idea of constructing a reward-weighted target distribution and projecting the policy toward it has deep roots in the RL-as-inference literature, which casts control as inference under a KL-regularized objective~\citep{dayan1997em, ziebart2010maxent, levine2018rl, geist2019regularized}. This perspective gives rise to a natural trust-region structure~\citep{amari1998natural, kakade2001npg, schulman2017trpo} and underlies a wide range of practical algorithms~\citep{peters2010reps,abdolmaleki2018mpo,song2019vmpo,peng2019awr,haarnoja2018sac,tomar2020mdpo}. However, these methods typically operate in continuous action spaces and resort to \textit{pointwise} projections, i.e., $-\sum_k w_k^{\ast}\log\pi_\theta(y_k)$.
In contrast, sampled responses in LLM form a finite response simplex, where normalization is exact and the partition function reduces to a finite sum over samples.
This structure enables listwise projection on the simplex, as exploited by LPO, which couples all responses through shared normalization and inherits satisfying gradients (Appendix~\ref{app:mpo}).

\paragraph{Listwise formulation.}
Listwise formulation has a long history in classical choice and learning-to-rank models~\citep{luce1959,plackett1975,cao2007listnet}, where a distribution over candidate sets or permutations is modeled or optimized.
Recent LLM alignment methods, such as DPO~\citep{rafailov2024dpo} and its extensions~\citep{liu2025lipo}, adopt pairwise or listwise preference structures to model relative comparisons among responses.
Listwise structures have also been employed in multi-agent LLM collaboration~\citep{yang2025maestro}.
Our approach operates in the standard RLVR setting for LLM post-training and explicitly constructs a target distribution based on verifiable rewards on the response simplex, followed by direct projection onto it.

\section{Proofs}
\label{app:proofs_main}

\subsection{KL Gradient Derivations}
\label{app:kl_gradients}

We derive the gradients of the forward and reverse KL divergences stated in Section~\ref{sec:projection}. For both derivations, we recall from Eq.~\ref{eq:P} that the logits are defined as $s_{\theta,k} = \log\pi_\theta(y_k \vert x) - \log\pi_b(y_k \vert x)$. Since the behavior policy $\pi_b$ is frozen, the gradient of the logit with respect to the parameters is simply $\nabla_\theta s_{\theta,k} = \nabla_\theta \log\pi_\theta(y_k \vert x)$.

\paragraph{Forward KL: $D_{\mathrm{KL}}(w^{\ast} \Vert P_\theta)$.}
By definition, $D_{\mathrm{KL}}(w^{\ast} \Vert P_\theta) = -\sum_{k=1}^K w_k^{\ast} \log P_{\theta,k} - H(w^{\ast})$, where $H(w^{\ast})$ is the entropy of $w^{\ast}$, which is constant with respect to $\theta$. Using the fact that $\log P_{\theta,k} = s_{\theta,k} - \log\sum_{j=1}^K e^{s_{\theta,j}}$, the Jacobian of the log-softmax is given by
\begin{equation}
\nabla_\theta \log P_{\theta,k} = \nabla_\theta s_{\theta,k} - \sum_{j=1}^K P_{\theta,j}\,\nabla_\theta s_{\theta,j}.
\end{equation}
Substituting this into the gradient of the forward KL divergence, we obtain:
\begin{align*}
\nabla_\theta D_{\mathrm{KL}}(w^{\ast} \Vert P_\theta) &= -\sum_{k=1}^K w_k^{\ast}\!\left(\nabla_\theta s_{\theta,k} - \sum_{j=1}^K P_{\theta,j}\,\nabla_\theta s_{\theta,j}\right) \\
&= -\sum_{k=1}^K w_k^{\ast}\,\nabla_\theta s_{\theta,k} + \left(\sum_{k=1}^K w_k^{\ast}\right) \sum_{j=1}^K P_{\theta,j}\,\nabla_\theta s_{\theta,j}.
\end{align*}
Since $w^{\ast}$ is a valid probability distribution ($\sum_{k=1}^K w_k^{\ast} = 1$), the second term simplifies. Reindexing the summation and substituting $\nabla_\theta s_{\theta,k} = \nabla_\theta \log\pi_\theta(y_k \vert x)$, we get:
\begin{equation}
\nabla_\theta D_{\mathrm{KL}}(w^{\ast} \Vert P_\theta) = \sum_{k=1}^K (P_{\theta,k} - w_k^{\ast})\,\nabla_\theta\log\pi_\theta(y_k \vert x). 
\end{equation}

\paragraph{Reverse KL: $D_{\mathrm{KL}}(P_\theta \Vert w^{\ast})$.}
Write the reverse KL as $D_{\mathrm{KL}}(P_\theta \Vert w^{\ast}) = \sum_{k=1}^K P_{\theta,k} \log(P_{\theta,k} / w_k^{\ast})$. We first compute the partial derivative with respect to a single logit $s_{\theta,j}$. Using the standard softmax Jacobian $\partial P_{\theta,k} / \partial s_{\theta,j} = P_{\theta,k}(\delta_{kj} - P_{\theta,j})$ and applying the product rule, we have:
\begin{align*}
\frac{\partial}{\partial s_{\theta,j}} D_{\mathrm{KL}}(P_\theta \Vert w^{\ast}) &= \sum_{k=1}^K P_{\theta,k}(\delta_{kj} - P_{\theta,j})\bigl[\log(P_{\theta,k} / w_k^{\ast}) + 1\bigr] \\
&= P_{\theta,j}\bigl[\log(P_{\theta,j} / w_j^{\ast}) + 1\bigr] - P_{\theta,j} \sum_{k=1}^K P_{\theta,k}\bigl[\log(P_{\theta,k} / w_k^{\ast}) + 1\bigr] \\
&= P_{\theta,j}\log\frac{P_{\theta,j}}{w_j^{\ast}} + P_{\theta,j} - P_{\theta,j} D_{\mathrm{KL}}(P_\theta \Vert w^{\ast}) - P_{\theta,j}(1) \\
&= P_{\theta,j}\Bigl[\log\frac{P_{\theta,j}}{w_j^{\ast}} - D_{\mathrm{KL}}(P_\theta \Vert w^{\ast})\Bigr].
\label{eq:kl_grad_general}
\end{align*}
Applying the multivariate chain rule $\nabla_\theta D_{\mathrm{KL}} = \sum_{j=1}^K \frac{\partial D_{\mathrm{KL}}}{\partial s_{\theta,j}} \nabla_\theta s_{\theta,j}$, we arrive at the full gradient:
\begin{equation}
\nabla_\theta\,D_{\mathrm{KL}}(P_\theta \Vert w^{\ast}) = \sum_{j=1}^K P_{\theta,j}\Bigl[\log\frac{P_{\theta,j}}{w_j^{\ast}} - D_{\mathrm{KL}}(P_\theta \Vert w^{\ast})\Bigr]\,\nabla_\theta\log\pi_\theta(y_j \vert x). \qed
\end{equation}

\paragraph{Logit-gap simplification for reverse KL.}
The per-logit coefficient for the reverse KL gradient, $c_k^{\mathrm{rev}} = \frac{\partial}{\partial s_{\theta,k}} D_{\mathrm{KL}}(P_\theta \Vert w^{\ast})$, can be elegantly simplified when written in terms of the logit gap $d_k = s_{\theta,k} - \phi_k$, where $\phi_k$ is the target logit from Eq.~\ref{eq:W}.

Express both probabilities explicitly with their partition functions: $P_{\theta,k} = \exp(s_{\theta,k})/Z_s$ and $w_k^{\ast} = \exp(\phi_k)/Z_\phi$. The log-ratio becomes:
\begin{equation}
\log\frac{P_{\theta,k}}{w_k^{\ast}} = (s_{\theta,k} - \log Z_s) - (\phi_k - \log Z_\phi) = (s_{\theta,k} - \phi_k) - (\log Z_s - \log Z_\phi) = d_k - c_s,
\end{equation}
where $c_s = \log Z_s - \log Z_\phi$ is strictly constant across all $k$. 
Consequently, the KL divergence can be written in terms of the expected gap:
\begin{equation}
D_{\mathrm{KL}}(P_\theta \Vert w^{\ast}) = \sum_{k=1}^K P_{\theta,k}(d_k - c_s) = \bar{d} - c_s,
\end{equation}
where $\bar{d} = \sum_{k=1}^K P_{\theta,k} d_k$. Substituting these back into the coefficient $c_k^{\mathrm{rev}}$, the constant $c_s$ perfectly cancels out:
\begin{equation}
c_k^{\mathrm{rev}} = P_{\theta,k}\bigl[(d_k - c_s) - (\bar{d} - c_s)\bigr] = P_{\theta,k}(d_k - \bar{d}).
\label{eq:revklcoeff}
\end{equation}
This reveals that the reverse KL gradient admits a baseline-subtracted form, 
where the baseline corresponds to the expected logit gap under the current policy $P_\theta$.

\subsection{Proof of Proposition~\ref{prop:pg_revkl}}
\label{app:proof_pg_revkl}

\propPgRevkl*
\begin{proof}
By the logit-gap simplification derived in Eq.~\ref{eq:revklcoeff}, the reverse KL gradient is fully characterized by the per-logit coefficients:
\begin{equation}
c_k^{\mathrm{rev}} = P_{\theta,k}(d_k - \bar{d}),
\end{equation}
where $d_k = s_{\theta,k} - A_k$ and $\bar{d} = \sum_{k=1}^K P_{\theta,k} d_k$. We directly evaluate this coefficient at the on-policy point $\pi_\theta = \pi_b$.

At the on-policy point, the logit offsets vanish ($s_{\theta,k} = 0$ for all $k$), which yields a uniform probability distribution over the generated list: $P_{\theta,k} = \softmax(\mathbf{0})_k = 1/K$. 
Consequently, the logit gap simplifies to $d_k = -A_k$. 

Applying the zero-mean advantage assumption $\sum_{k=1}^K A_k = 0$, the expected gap identically vanishes:
\begin{equation}
\bar{d} = \frac{1}{K} \sum_{k=1}^K (-A_k) = 0.
\end{equation}
Substituting these on-policy values back into the coefficient yields:
\begin{equation}
c_k^{\mathrm{rev}}\bigg|_{\pi_\theta=\pi_b} = \frac{1}{K} (-A_k - 0) = -\frac{A_k}{K}.
\end{equation}

Recall that the standard policy gradient in Eq.~\ref{eq:PG} can be expressed as $g_{\mathrm{PG}} = \sum_k c_k^{\mathrm{PG}} \nabla_\theta \log\pi_\theta(y_k \vert x)$ with coefficients $c_k^{\mathrm{PG}} = A_k / K$. 
Comparing the coefficients, we immediately obtain $c_k^{\mathrm{PG}} = -c_k^{\mathrm{rev}}|_{\pi_\theta=\pi_b}$, which proves that
\begin{equation}
g_{\mathrm{PG}} = -\nabla_\theta D_{\mathrm{KL}}(P_\theta \Vert w^{\ast})\big|_{\pi_\theta=\pi_b}.
\end{equation}
confirming that the policy gradient step is a gradient descent step on the reverse KL divergence at the on-policy point.

The centering assumption $\sum_k A_k = 0$ is without loss of generality: by the shift-invariance of softmax, replacing $A$ with $A - \bar{A}$ does not change the target $w^{\ast} = \softmax(A)$.

Finally, we clarify that the zero-mean assumption, i.e., $\sum_k A_k = 0$, is not a restrictive algorithmic requirement, but rather a natural reflection of the listwise projection's intrinsic mechanics. 
Due to the shift-invariance of the target softmax, any prompt-level scalar baseline applied uniformly to the group's rewards, e.g., the greedy baseline in ReMax~\citep{li2023remax}, is mathematically absorbed and nullified. 
The zero-sum constraint of the local simplex inherently induces a dynamically weighted mean-centering control variate ($d_k - \bar{d}$) during the gradient computation. 
At the on-policy point, this natively recovers the arithmetic zero-mean counterpart ($A_k - \bar{A}$). 
Thus, assuming centered advantages simply aligns our notation with the framework's built-in behavior at the exact point of equivalence.

\end{proof}

\paragraph{Off-policy approximation error.}
Proposition~\ref{prop:pg_revkl} establishes exact equality at $\pi_\theta = \pi_b$. We now quantify the discrepancy off-policy, incorporating the importance sampling ratio $r_k = \pi_\theta(y_k|x)/\pi_b(y_k|x)$ standard in practical PG methods without clipping.

Let $s_{\theta,k} = \log r_k$, $P_\theta = \softmax(s_\theta)$, and $\bar\delta = \max_k |r_k - 1|$. Both the policy gradient and the reverse KL gradient can be written as $\sum_k c_k \nabla_\theta \log\pi_\theta(y_k|x)$ with respective coefficients
\begin{equation}
c_k^{\mathrm{PG}} = \frac{r_k A_k}{K}, \qquad c_k^{\mathrm{revKL}} = -P_{\theta,k}(d_k - \bar{d}),
\end{equation}
where $d_k = s_{\theta,k} - A_k$ is the logit gap from Eq.~\ref{eq:revklcoeff} and $\bar{d} = \sum_k P_{\theta,k} d_k$. The per-coefficient discrepancy is
\begin{equation}
\Delta_k \;=\; c_k^{\mathrm{PG}} - c_k^{\mathrm{revKL}} \;=\; \frac{r_k A_k}{K} + P_{\theta,k}(d_k - \bar{d}),
\end{equation}
which vanishes identically at on-policy point ($\pi_\theta = \pi_b$).

We analyze the local regime where $\bar\delta < 1/2$, under which we have $r_k \in [1/2, 3/2]$, and thus 
$\|s_\theta\|_\infty\le 2\bar\delta$.

A first-order Taylor expansion of $P_{\theta,k} = \softmax(s_\theta)_k$ and $r_k = \exp(s_{\theta,k})$ around the zero vector $s_\theta = \mathbf{0}$ gives
\begin{equation}
P_{\theta,k} = \frac{1}{K} + \frac{s_{\theta,k} - \bar{s}_\theta}{K}
+ O\!\left(\frac{\|s_\theta\|_\infty^2}{K}\right), \quad r_k = 1 + s_{\theta,k} + O(\|s_\theta\|_\infty^2),
\end{equation}
where $\bar{s}_\theta = \frac{1}{K}\sum_k s_{\theta,k}$. Using the zero-mean advantage assumption $\sum_k A_k = 0$, the first-order expansion of $\bar{d} = \sum_k P_{\theta,k}(s_{\theta,k} - A_k)$ yields 
\begin{equation}
    \bar{d} = \bar{s}_\theta - \frac{1}{K}\sum_m s_{\theta,m} A_m + O(\bar\delta^2).
\end{equation}

Collecting terms:
\begin{equation}
\Delta_k = \frac{1}{K}\Bigl[(s_{\theta,k} - \bar{s}_\theta) + \bar{s}_\theta A_k + \frac{1}{K}\textstyle\sum_{m} s_{\theta,m} A_m\Bigr] + O\!\left(\frac{\bar\delta^2(1 + \|A\|_\infty)}{K}\right).
\end{equation}
Bounding the three terms via $|s_{\theta,k} - \bar{s}_\theta| \le 2\|s_\theta\|_\infty \le 4\bar\delta$, and symmetrically for the advantage terms $|\bar{s}_\theta A_k| \le 2\bar\delta\|A\|_\infty$ and $|\frac{1}{K}\sum_m s_{\theta,m} A_m| \le \|s_\theta\|_\infty \|A\|_\infty \le 2\bar\delta\|A\|_\infty$:
\begin{equation}
|\Delta_k| \;\le\; \frac{C\,\bar\delta\,(1 + \|A\|_\infty)}{K}
\label{eq:offpolicy_bound}
\end{equation}
for a universal constant $C > 0$, to leading order in $\bar\delta$. By the triangle inequality, the parameter-space gradient discrepancy satisfies
\begin{equation}
\|g_{\mathrm{PG}} - g_{\mathrm{revKL}}\| \;\le\; \textstyle\sum_k |\Delta_k|\,G_{\max} \;\le\; C'\bar\delta(1 + \|A\|_\infty)\,G_{\max},
\end{equation}
where $G_{\max} = \max_k \|\nabla_\theta \log\pi_\theta(y_k|x)\|$. The error is linear in the off-policy drift $\bar\delta$ and vanishes at the on-policy point, confirming that the policy gradient approximates reverse KL projection only in a neighborhood of the sampling distribution.
This rapid degradation under off-policy drift directly motivates the exact listwise projection proposed in LPO.

\paragraph{Remark: Connection to group-based policy gradients.}
Our off-policy analysis explicitly uncovers the structural relationship between exact listwise projection and practical group-based PG methods. By performing a strict first-order Taylor expansion on the exact reverse KL projection coefficient, we decouple it into three components:
\begin{equation}
c_k^{\mathrm{revKL}} \approx \underbrace{\frac{A_k}{K} + \frac{s_{\theta,k} A_k}{K}}_{\text{Pointwise advantage fitting}} - \underbrace{\left( \frac{\bar{s}_\theta A_k}{K} + \frac{1}{K^2}\textstyle\sum_{m} s_{\theta,m} A_m \right)}_{\text{Listwise normalization}} - \underbrace{\frac{s_{\theta,k} - \bar{s}_\theta}{K}}_{\text{Intrinsic entropy regularization}}.
\end{equation}
Remarkably, the gradient coefficient of group-based policy gradients without clipping, given by $c_k^{\mathrm{PG}} = \frac{r_k A_k}{K}$, yields a first-order expansion $c_k^{\mathrm{PG}} \approx \frac{A_k}{K} + \frac{s_{\theta,k} A_k}{K}$. 
This reveals a direct mathematical connection: the pointwise IS objective utilized in group-based policy gradients formally corresponds to the first-order advantage-fitting component of the reverse KL projection on the simplex, while the exact listwise formulation explicitly retains the coupled listwise normalization and intrinsic entropy regularization.

\subsection{Proof of Theorem~\ref{thm:gibbs}}
\label{app:proof_gibbs}

\thmGibbs*

\begin{proof}
Consider the optimization problem
\begin{equation}
w^{\ast}=\arg\max_{w\in\Delta^{K-1}} \hat J(w),
\qquad
\hat J(w)=\sum_{k=1}^K w_k R_k-\tau D_{\mathrm{KL}}(w\Vert P_t),
\end{equation}
where $P_t=\softmax(s_t)$ satisfies $P_{t,k}>0$ for all $k$.

Expanding the KL term gives
\begin{equation}
\hat J(w)
=
\sum_{k=1}^K w_k R_k
-\tau \sum_{k=1}^K w_k \log\frac{w_k}{P_{t,k}}.
\end{equation}
Introduce the Lagrangian for the simplex constraint $\sum_k w_k=1$:
\begin{equation}
\mathcal L(w,\lambda)
=
\sum_{k=1}^K w_k R_k
-\tau \sum_{k=1}^K w_k \log\frac{w_k}{P_{t,k}}
+\lambda\!\left(1-\sum_{k=1}^K w_k\right).
\end{equation}

Setting the stationary condition $\partial \mathcal L/\partial w_k=0$ yields
\begin{equation}
R_k-\tau\!\left(\log w_k-\log P_{t,k}+1\right)-\lambda=0,
\end{equation}
hence
\begin{equation}
\log w_k
=
\frac{R_k}{\tau}+\log P_{t,k}-1-\frac{\lambda}{\tau}.
\end{equation}
Exponentiating,
\begin{equation}
w_k
=
P_{t,k}\exp(R_k/\tau)\cdot C,
\qquad
C=\exp\!\left(-1-\lambda/\tau\right).
\end{equation}
Using $\sum_k w_k=1$, the normalization constant is
\begin{equation}
C^{-1}=\sum_{j=1}^K P_{t,j}\exp(R_j/\tau),
\end{equation}
therefore
\begin{equation}
w_k^\ast
=
\frac{P_{t,k}\exp(R_k/\tau)}
{\sum_{j=1}^K P_{t,j}\exp(R_j/\tau)}.
\label{eq:gibbs_appendix}
\end{equation}

Equivalently,
\begin{equation}
w^\ast=\softmax\!\bigl(R/\tau+\log P_t\bigr).
\end{equation}
Since $\log P_t=s_t-\log\sum_j e^{s_{t,j}}$, softmax is shift-invariant, this yields
\begin{equation}
w^\ast=\softmax(R/\tau+s_t)=\softmax(\phi),
\qquad
\phi_k=\frac{R_k}{\tau}+s_{t,k},
\end{equation}
which is Eq.~\ref{eq:W}.

Finally, $\hat J(w)$ is strictly concave on $\Delta^{K-1}$: the reward term is linear in $w$, while $D_{\mathrm{KL}}(w\Vert P_t)$ is strictly convex for $P_{t,k}>0$. Hence the maximizer is unique.
\end{proof}

\subsection{Proximal Objective as Reverse KL}
\label{app:proof_proximal_revkl}

\begin{restatable}[Proximal objective as reverse KL]{proposition}{propProximalRevkl}
\label{prop:proximal_revkl}
$\hat{J}(P_\theta) = -\tau D_{\mathrm{KL}}(P_\theta \| w^{\ast}) + \tau\log\hat{Z}$, so $\arg\max_{P_\theta} \hat{J}(P_\theta) = \arg\min_{P_\theta} D_{\mathrm{KL}}(P_\theta \| w^{\ast})$.
\end{restatable}
\begin{proof}
From Theorem~\ref{thm:gibbs}, $w_k^{\ast} = P_{t,k}\exp(R_k/\tau)/\hat{Z}$ where $\hat{Z} = \sum_j P_{t,j}\exp(R_j/\tau)$. Therefore $\log w_k^{\ast} = R_k/\tau + \log P_{t,k} - \log\hat{Z}$. Expanding the reverse KL:
\begin{align*}
-\tau D_{\mathrm{KL}}(P_\theta \| w^{\ast}) &= -\tau \sum_{k=1}^K P_{\theta,k} \log\frac{P_{\theta,k}}{w_k^{\ast}} \\
&= -\tau \sum_k P_{\theta,k} \bigl[\log P_{\theta,k} - \log w_k^{\ast}\bigr] \\
&= -\tau \sum_k P_{\theta,k} \bigl[\log P_{\theta,k} - R_k/\tau - \log P_{t,k} + \log\hat{Z}\bigr] \\
&= \sum_k P_{\theta,k} R_k - \tau \sum_k P_{\theta,k} \log\frac{P_{\theta,k}}{P_{t,k}} - \tau\log\hat{Z}.
\end{align*}
Recognizing $\hat{J}(P_\theta) = \sum_k P_{\theta,k} R_k - \tau D_{\mathrm{KL}}(P_\theta \| P_t) = \sum_k P_{\theta,k} R_k - \tau \sum_k P_{\theta,k} \log(P_{\theta,k}/P_{t,k})$, we obtain $\hat{J}(P_\theta) = -\tau D_{\mathrm{KL}}(P_\theta \| w^{\ast}) + \tau\log\hat{Z}$.
\end{proof}

\subsection{Proof of Theorem~\ref{thm:monotonic}}
\label{app:proof_monotonic}

\thmMonotonic*
\begin{proof}
\textbf{(a)} By Proposition~\ref{prop:proximal_revkl}, $\hat{J}(w) = -\tau D_{\mathrm{KL}}(w \| w^{\ast}) + \tau \log \hat{Z}$. Evaluating at the anchor $P_t$:
\begin{equation}
\hat{R}(P_t) = \hat{J}(P_t) = \tau \log \hat{Z} - \tau D_{\mathrm{KL}}(P_t \| w^{\ast}).
\end{equation}
Evaluating at the target $w^{\ast}$ (where $D_{\mathrm{KL}}(w^{\ast}\|w^{\ast}) = 0$):
\begin{equation}
\hat{R}(w^{\ast}) = \hat{J}(w^{\ast}) + \tau D_{\mathrm{KL}}(w^{\ast}\|P_t) = \tau\log\hat{Z} + \tau D_{\mathrm{KL}}(w^{\ast}\|P_t).
\end{equation}
Subtracting: $\hat{R}(w^{\ast}) - \hat{R}(P_t) = \tau[D_{\mathrm{KL}}(w^{\ast}\|P_t) + D_{\mathrm{KL}}(P_t\|w^{\ast})]$.

\textbf{(b)} We bound the expected reward error using the Total Variation (TV) distance. By definition, the $L_1$ norm relates to the TV distance as $\|P_{t+1} - w^{\ast}\|_1 = 2\TV(P_{t+1}, w^{\ast})$. 
By applying Pinsker's inequality, the TV distance is upper-bounded by either choice of KL projection:
\begin{equation}
\TV(P_{t+1}, w^{\ast}) \le \sqrt{\frac{1}{2}\min\bigl(D_{\mathrm{KL}}(w^{\ast} \Vert P_{t+1}),\; D_{\mathrm{KL}}(P_{t+1} \Vert w^{\ast})\bigr)}.
\end{equation}
Assuming the projection step achieves $\TV(P_{t+1}, w^{\ast}) \le \epsilon_{\mathrm{proj}}$, we apply H\"older's inequality with $|R_k| \le R_{\max}$:
\begin{align}
|\hat{R}(P_{t+1}) - \hat{R}(w^{\ast})| = \Bigl|\sum_k (P_{t+1,k} - w_k^{\ast}) R_k\Bigr|  
\le R_{\max}\|P_{t+1} - w^{\ast}\|_1 = 2 R_{\max} \epsilon_{\mathrm{proj}}.
\end{align}
Substituting this error term back into the minorization inequality from part (a) yields the final bound.

\textbf{(c)} Combining (a) and (b): 
\begin{align*}
    \hat{R}(P_{t+1}) &\ge \hat{R}(w^{\ast}) - 2R_{\max}\epsilon_{\mathrm{proj}} 
    \\
    &\ge \hat{R}(P_t) + \tau[D_{\mathrm{KL}}(w^{\ast}\|P_t) + D_{\mathrm{KL}}(P_t\|w^{\ast})] - 2R_{\max}\epsilon_{\mathrm{proj}}.
\end{align*}

\end{proof}

\subsection{Proof of Proposition~\ref{prop:optimality}}
\label{app:proof_optimality}

\propOptimality*
\begin{proof}
By induction: the base case $t=0$ is trivial. If 
$\pi_t(y) \propto \pi_0(y)\exp(tR(y)/\tau)$, then 
$\pi_{t+1}(y) \propto \pi_t(y)\exp(R(y)/\tau) 
\propto \pi_0(y)\exp((t+1)R(y)/\tau)$.

For convergence, consider any two responses $y_1, y_2$ with $R(y_1) > R(y_2)$:
\begin{equation}
\frac{\pi_t(y_1)}{\pi_t(y_2)} 
= \frac{\pi_0(y_1)}{\pi_0(y_2)}\ \exp\bigl(t\cdot\frac{R(y_1){-}R(y_2)}{\tau}\bigr) 
\to \infty.
\end{equation}
Since $\pi_0(y) > 0$ for all $y$, the mass concentrates on 
$\arg\max_y R(y)$, giving 
$\mathbb{E}_{\pi_t}[R] \to \max_y R(y)$.
\end{proof}

\paragraph{Connecting global optimality to LPO.}
Proposition~\ref{prop:optimality} characterizes the ideal full-space proximal operator: if one could exactly apply the Gibbs update over the entire response space, the resulting iteration converges to the global RL optimum.
For autoregressive LLMs, however, the required partition function is intractable over the combinatorially large sequence space. This computational barrier motivates LPO. Rather than operating in the full space, LPO restricts the same target-projection principle to the finite response simplex induced by $K$ sampled trajectories, yielding a principled and fully tractable approximation to the ideal proximal step.

\subsection{Proof of Corollary~\ref{cor:bounded} }
\label{app:proof_bounded}

\corBounded*
\begin{proof}
Let $c_k^{\mathrm{fwd}} = P_{\theta,k} - w_k^{\ast}$ where $P_\theta, w^{\ast} \in \Delta^{K-1}$.

\textbf{(a)} Since $P_{\theta,k} \in [0,1]$ and $w_k^{\ast} \in [0,1]$, we have $c_k^{\mathrm{fwd}} \in [-1, 1]$, hence $|c_k^{\mathrm{fwd}}| \le 1$.

\textbf{(b)} Since both $P_\theta$ and $w^{\ast}$ are probability distributions, $\sum_{k=1}^K c_k^{\mathrm{fwd}} = \sum_k P_{\theta,k} - \sum_k w_k^{\ast} = 1 - 1 = 0$. Partitioning into positive and negative parts: $\sum_{c_k^{\mathrm{fwd}} > 0} c_k^{\mathrm{fwd}} = -\sum_{c_k^{\mathrm{fwd}} < 0} c_k^{\mathrm{fwd}}$. Therefore $\sum_k |c_k^{\mathrm{fwd}}| = 2\sum_{c_k^{\mathrm{fwd}} > 0} c_k^{\mathrm{fwd}}$. Since each $c_k^{\mathrm{fwd}} \le 1$ and the positive parts sum to at most 1 (because $\sum_{c_k^{\mathrm{fwd}}>0} c_k^{\mathrm{fwd}} \le \sum_{c_k^{\mathrm{fwd}}>0} P_{\theta,k} \le 1$), we obtain $\sum_k |c_k^{\mathrm{fwd}}| \le 2$.

\textbf{(c)} As $P_\theta \to w^{\ast}$, $c_k^{\mathrm{fwd}} = P_{\theta,k} - w_k^{\ast} \to 0$ for all $k$ by definition.

For the parameter-space bound, $\nabla_\theta \mathcal{L}_{\mathrm{LPO_{fwd}}} = \sum_k c_k^{\mathrm{fwd}} \nabla_\theta \log\pi_\theta(y_k|x)$, so by the triangle inequality: $\|\nabla_\theta \mathcal{L}_{\mathrm{LPO_{fwd}}}\| \le \sum_k |c_k^{\mathrm{fwd}}| \cdot \|\nabla_\theta \log\pi_\theta(y_k|x)\| \le 2G_{\max}$.
\end{proof}

\subsection{Proof of Corollary~\ref{cor:mode_coverage}}
\label{app:proof_mode_coverage}

\corModeCoverage*
\begin{proof}
To rigorously bound $P_{\theta,k}$, we construct a binary event space (whether a response is $k$ or not $k$). By the Data Processing Inequality (DPI), the binary KL divergence is bounded by the full KL divergence:
\begin{equation}
w_k^{\ast}\log\frac{w_k^{\ast}}{P_{\theta,k}} + (1-w_k^{\ast})\log\frac{1-w_k^{\ast}}{1-P_{\theta,k}} \le D_{\mathrm{KL}}(w^{\ast} \| P_\theta) \le D.
\end{equation}

Since $1-P_{\theta,k} \le 1$, the term $\log(1/(1-P_{\theta,k})) \ge 0$. Dropping this non-negative component preserves the upper bound inequality:
\begin{equation}
w_k^{\ast}\log\frac{w_k^{\ast}}{P_{\theta,k}} + (1-w_k^{\ast})\log(1-w_k^{\ast}) \le D.
\end{equation}

Rearranging the terms to isolate $P_{\theta,k}$, we obtain:
\begin{equation}
\log\frac{w_k^{\ast}}{P_{\theta,k}} \le \frac{D}{w_k^{\ast}} - \frac{1-w_k^{\ast}}{w_k^{\ast}}\log(1-w_k^{\ast}).
\end{equation}

Exponentiating both sides yields the rigorously corrected lower bound:
\begin{equation}
P_{\theta,k} \ge w_k^{\ast} \exp\left(-\frac{D}{w_k^{\ast}}\right) (1-w_k^{\ast})^{\frac{1-w_k^{\ast}}{w_k^{\ast}}}.
\end{equation}

Let $f(x) = x \exp(-D/x) (1-x)^{\frac{1-x}{x}}$. It can be shown that $f(x)$ is monotonically increasing for $x \in (0, 1)$. Given the assumption that $w_k^{\ast} \ge \alpha > 0$, it follows that $P_{\theta,k} \ge f(w_k^{\ast}) \ge f(\alpha)$. Therefore, we conclude:
\begin{equation}
P_{\theta,k} \ge \alpha \exp\left(-\frac{D}{\alpha}\right) (1-\alpha)^{\frac{1-\alpha}{\alpha}}\ge \alpha \exp\left(-\frac{D}{\alpha}-1\right).
\end{equation}

\end{proof}

\section{Additional Discussions}
\label{app:additional}

\subsection{Contribution Clarification}
\label{app:scope}

This work mainly makes two contributions: (i) the Target-Projection (TP) framework, a unified geometric interpretation showing that dominant group-based PG methods implicitly construct the same Gibbs target family and approximate a reverse KL projection toward it; and (ii) Listwise Policy Optimization (LPO), which makes this target-projection explicit and, by decoupling the target from the projection, opens divergence selection as a new design axis inaccessible under the implicit PG paradigm, with provable theoretical guarantees and consistent empirical gains.

Several clarifications are included regarding the scope and design choices.
\begin{enumerate}[leftmargin=15pt]
    \item The TP analysis and LPO operate in the group-based regime ($K \ge 2$), which covers the vast majority of contemporary RLVR practice. Single-sample methods ($K{=}1$) lack a per-prompt simplex and require a different analytical treatment, as discussed in Appendix~\ref{app:extensions}.
    \item The specific choice of forward or reverse KL is not a core contribution of this work, with the broader design space discussed in Appendix~\ref{app:general_divergence}. Reverse KL is a natural choice since policy gradient implicitly performs an approximate reverse KL projection (Proposition~\ref{prop:pg_revkl}). Forward KL is similarly motivated by its mode-covering geometry, whose benefit for diversity has been observed in adjacent settings~\citep{wang2023beyond, li2025choice, anthonyreverse}.
    \item Recent engineering innovations, e.g., dynamic sampling and asymmetric clipping~\citep{yu2025dapo}, are orthogonal to LPO and can be combined with it, as discussed in Appendix~\ref{app:dapo}. The current experiments intentionally use a minimal shared pipeline to cleanly attribute gains to the target-projection mechanism.
    \item The experiments adopt a paired-temperature design, varing only the projection mechanism to isolate it as the sole controlled variable. While the temperature $\tau$ could be independently tuned, this is deliberately avoided to ensure a fair comparison, leaving for future work.
    \item All theoretical results and the implementation hold for arbitrary rewards $R_k \in \mathbb{R}$ without binary assumption. The focus on binary outcome rewards reflects the dominant RLVR setting, while the programming experiments already assess a non-binary reward.
    \item The experimental evaluation focuses on reasoning tasks with verifiable rewards, the primary application domain of group-based PG methods. The TP analysis and LPO are not inherently limited to this setting: extending empirical validation to broader RL post-training scenarios, e.g., RLHF with learned reward models, is a natural direction for future work.
\end{enumerate}

\subsection{Extensions and Future Directions}
\label{app:extensions}

\paragraph{Step-level listwise projection.}
Real-world applications often necessitate fine-grained optimization beyond sequence-level rewards, such as multi-turn agentic RL~\citep{jin2025searchr1} or reasoning tasks with dense step-level rewards~\citep{lightman2023let}.
The current sequence-level framework may extend to these scenarios: given a shared intermediate state, one can sample $K$ candidate continuations to form the local response simplex.
Crucially, deriving the target for these immediate steps requires estimating their expected final outcomes.
This can be achieved by rolling out each continuation to the terminal state.
Alternatively, to bypass the prohibitive cost of full rollouts, one can rely on a value network or a value-calibrated process reward model~\citep{lightman2023let} to estimate these expected future returns.
In either setting, the core LPO machinery carries over naturally, shifting the primary practical challenge to the fidelity of the step-level value estimation.

\paragraph{Off-policy replay.}
Because the listwise projection operates on the response simplex for each prompt, LPO can theoretically incorporate off-policy data to improve sample efficiency and amortize the high rollout costs typical of RLVR.
Specifically, by recording the behavior policy $\pi_b$ used to generate past responses, LPO can account for off-policy drift via importance sampling ratios $\pi_t/\pi_b$ and $\pi_\theta/\pi_b$. 
The listwise normalization implicitly acts as a self-normalized importance sampling (SNIS) estimator, inherently adapting the policy and target distribution without altering the underlying projection geometry.
Despite this theoretical elegance, realizing off-policy replay introduces practical optimization hurdles. 
As the policy evolves, severe drift from stale checkpoints can yield extreme probability ratios, which may collapse the effective listwise distributions and destabilize the projection gradients. 
Developing robust staleness-filtering or trust-region buffer management strategies to stabilize off-policy LPO remains a promising direction for future work.

\paragraph{Beyond group-based sampling.}
Current LPO requires $K \ge 2$ responses per prompt to form the response simplex, which precludes direct application in single-sample ($K{=}1$) pipelines.
One potential resolution is to assemble virtual groups using the off-policy replay buffer, though this inherits the aforementioned stability challenges.
A minimal alternative constructs a virtual response simplex by pairing the single sampled reward $R$ with a batch-level baseline $b$.
This contrastive formulation yields a sigmoid-squashed gradient coefficient $c = \tfrac{1}{2} - \sigma\!\bigl((R - b)/\tau\bigr)$ that preserves boundedness ($|c| \le 1/2$), though it necessarily sacrifices the zero-sum property as there is no physical group.
Both relaxations remain exploratory and characterizing their practical tradeoffs is a promising avenue.

\paragraph{Alternative divergences and adaptive scheduling.}
A distinctive feature of the explicit target-projection framework is the complete decoupling of the target distribution from the projection divergence, which is a critical design axis unavailable to policy gradient methods.
This separation naturally invites the exploration of entirely different statistical divergences that might induce unique and favorable optimization geometries tailored to specific reasoning tasks, as analyzed in Appendix~\ref{app:general_divergence}.
Furthermore, this decoupling enables dynamic scheduling strategies during training. 
For instance, one could employ forward KL in early stages to encourage broad exploration, and subsequently switch to reverse KL for stable late-stage exploitation, or progressively anneal the temperature $\tau$ to sharpen the target as the performance improves.
Systematic exploration of this expanded design space constitutes a natural next step for RL post-training.

\subsection{Existing Group-based RLVR as Implicit Target-Projection}
\label{app:dapo}

As revealed in Section~\ref{sec:analysis}, the dominant group-based RL algorithms can be unified under a shared geometric structure: each defines an implicit Gibbs target distribution $w^{\ast}$ and executes an approximate reverse KL projection via policy gradient.
The methods differ primarily in how they normalize advantages, which determines the implicit temperature $\tau$ and thus the sharpness of $w^{\ast}$.
Table~\ref{tab:tp_extended} groups these methods by their implicit target family.

\begin{table}[h]
\centering
\caption{Target-Projection decomposition of existing methods, grouped by implicit target family.}
\label{tab:tp_extended}
\resizebox{0.6\linewidth}{!}{
\begin{tabular}{llc}
\toprule
Target family & Methods & $\tau$\\
\midrule
$\softmax(R/\sigma_G)$ & GRPO, DAPO, CISPO, GSPO & $\sigma_G$  \\
$\softmax(R)$ & Dr.GRPO, RLOO ($\tau{\approx}1$), ReMax & $1$  \\
$\softmax(R/\mu_G)$ & MaxRL & $\mu_G$ \\
$\softmax(R/\sigma_{B'})$ & REINFORCE++$_{w/ Baseline}$ & $\sigma_{B'}$  \\
\bottomrule
\end{tabular}
}
\end{table}

\paragraph{$\sigma_G$-family: GRPO~\citep{shao2024grpo}, DAPO~\citep{yu2025dapo}, CISPO~\citep{chen2025cispo}, GSPO~\citep{zheng2025gspo}.}

\begin{equation}
A_k = \frac{R_k - \mu_G}{\sigma_G}, \qquad \tau = \sigma_G = \sqrt{\mu_G(1-\mu_G)}, \qquad w^{\ast} = \softmax(R/\sigma_G).
\end{equation}
The temperature is adaptive: maximal at $\mu_G = 0.5$ (balanced groups) and vanishing as $\mu_G \to 0$ or $1$ (near-unanimous groups), coupling target sharpness to group difficulty.
DAPO adds four projection-level innovations: asymmetric clipping, dynamic sampling to filter uninformative groups, token-level loss normalization, and overlong reward shaping.
CISPO modifies the projection by replacing clipping with a stop-gradient on the clipped importance ratio, preserving gradient contributions from all tokens. 
GSPO lifts the importance ratio and clipping from the token level to the sequence level $s_k = [\pi_\theta(y_k|x)/\pi_{\theta_{\mathrm{old}}}(y_k|x)]^{1/|y_k|}$, aligning the optimization unit with the reward granularity.
Many of these projection-level engineering tricks are orthogonal to our target construction and can be seamlessly integrated into the LPO framework.

\paragraph{$\tau{\approx}1$ family: Dr.GRPO~\citep{liu2025drgrpo}, RLOO~\citep{ahmadian2024rloo}, ReMax~\citep{li2023remax}.}
\begin{equation}
A_k = R_k - \mu, \qquad \tau = 1, \qquad w^{\ast} = \softmax(R).
\end{equation}
Dr.GRPO removes $\sigma_G$ normalization (fixing $\tau = 1$) and adopts token-level loss normalization to address length bias.
RLOO uses a leave-one-out baseline ($\tau = (K{-}1)/K \to 1$), yielding an unbiased advantage estimator with nearly the same implicit target.
ReMax uses a greedy-decode baseline $R_{\mathrm{greedy}}$ which cancels in the softmax, recovering the same target.

\paragraph{MaxRL~\citep{tajwar2026maxrl}.}
$A_k = (R_k{-}\mu_G)/\mu_G$, $\tau = \mu_G = n/K$, $w^\ast=\softmax(R/\mu_G)$. The temperature is directly proportional to the success rate, implementing an implicit curriculum: hard prompts (low $\mu_G$) receive aggressively sharp targets to encourage exploitation, while easy prompts (high $\mu_G$) receive diffuse targets to maintain diversity.

\paragraph{REINFORCE++~\citep{hu2025reinforce}.}
REINFORCE++ proposes two variants.
The base variant uses single-stage batch normalization $A_k = (R_k - \mu_B)/\sigma_B$; its primary use case is $K{=}1$ , where no per-prompt group exists and the target-projection decomposition does not apply.
The \emph{w/ Baseline} variant employs a two-stage process: first subtract the per-group mean to reshape rewards, $A'_k = R_k - \mu_G$, then normalize by the global batch statistics of these centered rewards, $A_k^{\mathrm{norm}} = (A'_k - \mu_{B'})/\sigma_{B'}$.
Since both $\mu_G$ and $\mu_{B'}$ are constant within a group, they cancel under softmax, yielding $w^{\ast} = \softmax(R/\sigma_{B'})$ with $\tau = \sigma_{B'}$.
Here $\sigma_{B'}$ is the batch-level standard deviation of the group-centered rewards.

\subsection{Listwise vs. Pointwise Projection}
\label{app:mpo}

An alternative to the listwise framework developed in Section~\ref{sec:lpo} is standard \emph{pointwise} projection, a paradigm widely used in classical RL algorithms (e.g., MPO~\citep{abdolmaleki2018mpo} and AWR~\citep{peng2019awr}). Both paradigms share the same target step, constructing the reward-weighted Gibbs distribution $w_k^{\ast} \propto \pi_b(y_k)\exp(R_k/\tau)$, but they diverge fundamentally in how they project the policy toward it.

\paragraph{Independent vs. coupled formulation.}
Pointwise projection minimizes a weighted negative log-likelihood:
\begin{equation}
\mathcal{L}_{\mathrm{pointwise}} = -\sum_{k=1}^K w_k^{\ast} \log \pi_\theta(y_k \vert x),
\end{equation}
which treats each sampled response independently. The gradient coefficient for response $k$ is simply $c_k^{\mathrm{point}} = -w_k^{\ast}$. This yields a strictly one-sided update that pushes probability mass \emph{toward} high-weight responses without any coupled counterbalancing force. 

In contrast, LPO with forward KL minimizes divergence $D_{\mathrm{KL}}(w^{\ast} \Vert P_\theta)$, where $P_\theta = \softmax(s_\theta)$ is the policy's listwise distribution. This explicit listwise formulation couples all $K$ responses through a shared normalization factor. Consequently, the gradient coefficient $c_k = P_{\theta,k} - w_k^{\ast}$ is strictly two-sided: responses where the policy over-allocates probability mass ($P_{\theta,k} > w_k^{\ast}$) are actively suppressed, while under-allocated responses are boosted.

\paragraph{Structural consequences.}
This architectural difference in the projection space produces three structural properties that pointwise projection inherently lacks:

\begin{enumerate}
\item \textbf{Zero-sum updates.} For LPO, the coefficients strictly sum to zero: $\sum_k c_k = 0$, acting as a built-in control variate for variance reduction. For pointwise projection, $\sum_k c_k^{\mathrm{point}} = -\sum_k w_k^{\ast} = -1$, yielding a net gradient direction that exerts a continuous, uncalibrated pull on the parameter space.

\item \textbf{Bounded gradients.} LPO coefficients satisfy $\sum_k |c_k| \le 2$ (Corollary~\ref{cor:bounded}), providing an intrinsic, reward-scale-invariant bound on the projection step. Pointwise projection lacks this relative scaling.

\item \textbf{Self-correcting convergence.} As $P_\theta \to w^{\ast}$, the LPO coefficients vanish ($c_k = P_{\theta,k} - w_k^{\ast} \to 0$), meaning optimization naturally terminates once the target is matched. Pointwise coefficients ($c_k^{\mathrm{point}} = -w_k^{\ast}$) are constant with respect to $\pi_\theta$.
\end{enumerate}

\paragraph{Origin of the difference.}
The pointwise objective $-\sum_k w_k^{\ast}\log\pi_\theta(y_k)$ mathematically corresponds to the cross-entropy $H(w^{\ast}, \pi_\theta)$, which equals $D_{\mathrm{KL}}(w^{\ast} \Vert \pi_\theta) + H(w^{\ast})$. Because $\pi_\theta$ is evaluated independently per response and is not normalized over the response group, this KL divergence measures the gap between unnormalized densities. LPO, by contrast, operates on the normalized listwise distribution $P_\theta \in \Delta^{K-1}$, which lives on the exact same finite probability simplex as $w^{\ast}$. This shared simplex geometry is what dictates the two-sided, zero-sum gradient structure.

\paragraph{Connection to Expectation-Maximization.}
The explicit target-projection procedure mirrors the structure of the Expectation-Maximization (EM) algorithm~\citep{dayan1997em,neal1998em}: the Gibbs target construction resembles an E-step that forms a target distribution, while the divergence minimization corresponds to an M-step that fits the model to this target.

\subsection{Connection to DPO and Preference Optimization}
\label{app:dpo}

When $K=2$, LPO reduces to a pairwise objective closely related to Direct Preference Optimization (DPO)~\citep{rafailov2024dpo}. Consider two responses: a preferred response $y_w$ with reward $R_w=1$, and a dispreferred response $y_l$ with reward $R_l=0$.

For two responses, the listwise distribution becomes
\begin{equation}
P_w=
\frac{\exp(s_w)}{\exp(s_w)+\exp(s_l)}
=
\sigma(s_w-s_l),
\end{equation}
where
$
s_k=\log(\pi_\theta(y_k|x)/\pi_b(y_k|x))
$
and $\sigma(\cdot)$ is the sigmoid function.
In on-policy setup ($\pi_t=\pi_b$), the baseline distribution is uniform, yielding
\begin{equation}
w_w^*=\sigma(1/\tau),\qquad
w_l^*=\sigma(-1/\tau).
\end{equation}

The forward-KL objective then simplifies to
\begin{equation}
\mathcal{L}_{\mathrm{LPO_{fwd}}}
=
-\sigma(1/\tau)\log \sigma(s_w-s_l)
-\sigma(-1/\tau)\log \sigma(s_l-s_w),
\end{equation}
which is a binary cross-entropy objective with temperature-controlled soft targets.

By comparison, DPO uses the pairwise logistic objective
\begin{equation}
\mathcal{L}_{\mathrm{DPO}}
=
-\log \sigma\!\bigl(\beta(s_w-s_l)\bigr),
\end{equation}
where
$
s_k=\log(\pi_\theta(y_k|x)/\pi_{\mathrm{ref}}(y_k|x)).
$

Thus, both methods share the same pairwise sigmoid structure, but differ fundamentally in four aspects: 
(i) standard DPO operates within an offline paradigm on static datasets, whereas LPO is an online RL algorithm; 
(ii) DPO penalizes logits against a static reference policy $\pi_{\mathrm{ref}}$, whereas LPO derives its target within a trust region around the pre-update policy $\pi_t$;
(iii) DPO is derived under a Bradley--Terry style preference model, whereas LPO arises from explicit divergence projection on the response simplex;
(iv) LPO uses soft targets controlled by $\tau$, recovering a hard preference target as $\tau\to0$.

This view places DPO-style pairwise optimization as the foundational $K=2$ baseline of the broader LPO framework, which naturally extends from pairwise preferences ($K=2$) to listwise optimization ($K>2$), and further to the population-level RL-as-inference limit as $K\to\infty$.

\paragraph{Distinction from Listwise Preference Optimization.}
Recent works like LiPO~\citep{liu2025lipo} extend DPO to listwise preference optimization with Plackett--Luce style ranking models.
Despite the similar ``listwise'' terminology, these methods learn from ranked preference data $y_1 \succ y_2 \succ \dots \succ y_K$.
In contrast, LPO is designed for online RLVR, utilizing absolute reward signals for explicit target-projection without any ranking-model assumptions.
Mathematically, Plackett--Luce ranking models and our Gibbs target both take a normalized softmax form.
Thus, they represent complementary ways of obtaining the same exponential-family target: one inferred from comparisons, the other directly specified by rewards.

\subsection{Extension to General Divergences}
\label{app:general_divergence}

In Section~\ref{sec:projection}, we instantiated LPO using forward and reverse KL divergences. However, the projection framework is not specific to KL and can be applied to any differentiable divergence defined on the probability simplex, including general $f$-divergences such as the Jensen--Shannon divergence.

Let $\mathcal{L} = D(w^\ast, P_\theta)$ be a differentiable divergence on $\Delta^{K-1}$. The gradient takes the form
\begin{equation}
\nabla_\theta \mathcal{L} = \sum_{k=1}^K c_k \nabla_\theta \log \pi_\theta(y_k \mid x),
\end{equation}
where the gradient coefficient is $c_k = \partial \mathcal{L} / \partial s_{\theta,k}$. Applying the chain rule with the softmax Jacobian, we have
\begin{equation}
c_k = \sum_{j=1}^K \frac{\partial \mathcal{L}}{\partial P_{\theta,j}} \frac{\partial P_{\theta,j}}{\partial s_{\theta,k}} = P_{\theta,k} \frac{\partial \mathcal{L}}{\partial P_{\theta,k}} - P_{\theta,k} \sum_{j=1}^K P_{\theta,j} \frac{\partial \mathcal{L}}{\partial P_{\theta,j}}.
\end{equation}
Summing these coefficients over all $K$ responses yields the identity
\begin{equation}
\sum_{k=1}^K c_k = \sum_{k=1}^K P_{\theta,k} \frac{\partial \mathcal{L}}{\partial P_{\theta,k}} - \underbrace{\left(\sum_{k=1}^K P_{\theta,k}\right)}_{=1} \left(\sum_{j=1}^K P_{\theta,j} \frac{\partial \mathcal{L}}{\partial P_{\theta,j}}\right) = 0.
\end{equation}

This zero-sum property is a direct consequence of the softmax parameterization on the probability simplex and holds for any differentiable objective defined on $P_\theta$. It plays a role analogous to a baseline in policy gradient methods and contributes to stabilizing the update.

While this zero-sum property is universal, other characteristics, such as coefficient boundedness or mode-seeking behavior, depend on the specific choice of divergence. 
KL divergences are adopted as natural default choices in LPO due to their stability and well-understood geometry.

\subsection{Entropy Regularization and Reverse KL Diversity}
\label{app:entropy}

\paragraph{Reverse KL as max-entropy RL.}
The objective of LPO with reverse KL is equivalent to $\max_\theta \sum_k P_{\theta,k} \phi_k + H(P_\theta)$, mirroring the maximum entropy RL objective~\citep{ziebart2010maxent}.
Here, the explicit entropy bonus emerges naturally from the structural formulation of the divergence.
In contrast, standard policy gradient methods lose this property, as they are equivalent to a first-order approximation only at the on-policy point.

\paragraph{Entropy regularization as target mixing.}
Adding entropy bonus $\gamma H(\pi_\theta)$ modifies the listwise target to $\tilde{w}^{\ast} = \softmax(R/(\tau + \gamma))$ in the on-policy setup, equivalent to increasing $\tau$ by $\gamma$. The entropy bonus is redundant when $\tau$ is a controllable hyperparameter.

\subsection{Broader Societal Impacts}
\label{appsec:broader}

This work introduces LPO as a novel paradigm for RLVR. As an algorithmic contribution to policy optimization, LPO may improve the efficiency and stability of RL post-training, potentially reducing the computational cost of training strong LLMs.
More broadly, improvements in reasoning capability and training efficiency may indirectly benefit downstream applications of LLMs, such as scientific problem solving, software development, and educational tools, by enabling more capable and reliable systems.
On the negative side, the method inherits the general risks associated with increasingly capable LLMs, including potential dual-use concerns if deployed without appropriate safeguards.
Additionally, while LPO improves optimization efficiency in RLVR, addressing the environmental and societal costs of large-scale model training remains an open challenge for the broader research community.

\section{Implementation Details}
\label{appsec:implementation}

\subsection{Tasks}
\subsubsection{Logical Reasoning}

\paragraph{Training Dataset.}
We adopt the Countdown Number Game as the logical reasoning testbed.
This task requires models to synthesize basic arithmetic operations to reach a target value using a provided set of integers.
We use a subset of 2000 problems sampled from the Countdown 34 dataset (\url{https://huggingface.co/datasets/Jiayi-Pan/Countdown-Tasks-3to4}) as training dataset, which supplies either three or four source numbers per question.

\paragraph{Evaluation Benchmarks.}
We assess model performance using two reserved evaluation sets: a split of 512 instances from Countdown 34 (CD34) and a subset of 512 instances from Countdown 4 (CD4), available at \url{https://huggingface.co/datasets/Jiayi-Pan/Countdown-Tasks-4}. The CD4 variant is notably more difficult because it strictly guarantees four source numbers per problem, thereby massively expanding the combinatorial search space.
To evaluate performance, we generate 64 independent responses per instance to compute both the expected \texttt{Pass@1} (the average correctness across all 64 samples) and the \texttt{Pass@64} metrics. All reported training curves reflect the average performance across both the CD34 and CD4 benchmarks.

\paragraph{Reward Function.}
Following \citet{tinyzero}, we augment the binary accuracy reward with a formatting bonus. This design explicitly incentivizes proper structural adherence alongside correct reasoning:
\begin{equation}
r =
\begin{cases}
1   & \text{if the response is correct}, \\
0.1 & \text{if the response is incorrect but properly formatted}, \\
0   & \text{otherwise}.
\end{cases}
\end{equation}

\subsubsection{Mathematics Reasoning}

\paragraph{Training Dataset.}
Following \citet{qu2025mopps}, we train models on the MATH dataset~\citep{hendrycksmath2021}, which consists of 7.5k problems from mathematics competitions.
We use the public version hosted at \url{https://huggingface.co/datasets/DigitalLearningGmbH/MATH-lighteval}.
For extended validation at a larger scale, we additionally train the Qwen3-14B-Base model on the Polaris~\citep{Polaris2025} dataset in Appendix~\ref{appsec:polaris}, which comprises a broader collection of roughly 53k high-quality mathematical reasoning problems.
The Polaris dataset is hosted at \url{https://huggingface.co/datasets/POLARIS-Project/Polaris-Dataset-53K}.

\paragraph{Evaluation Benchmarks.}
Following \citet{gao2025prompt, qu2026gps}, we evaluate mathematical reasoning performance on a suite of benchmarks, including AIME24, AIME25, AMC23, MATH500~\citep{lightman2023let}, Minerva Math~\citep{lewkowycz2022solving}, and OlympiadBench~\citep{he2024olympiadbench}, using the datasets hosted at \url{https://huggingface.co/datasets/math-ai}.
Following prior works~\citep{gao2025prompt, qu2026gps}, we sample $k$ independent responses per problem to compute both the expected \texttt{Pass@1} (defined as the average accuracy across all $k$ samples, or \texttt{avg@k}) and the \texttt{Pass@k} metrics. The sample size $k$ is tailored to the size and difficulty of each benchmark: we set $k=32$ for competition-level suites (AIME24, AIME25, AMC23), $k=4$ for Minerva Math, and $k=1$ for MATH500 and OlympiadBench. Training curves report the average performance across all math benchmarks.

\paragraph{Reward Function.}
Following the default configuration in verl~\citep{sheng2024verl}, we use a binary reward function that assigns $1$ to correct responses and $0$ otherwise.

\subsubsection{Programming}

\paragraph{Training Dataset.}
To assess the training performance in code generation, following \citet{cui2025prime}, we adopt the code split in PRIME dataset, available at \url{https://huggingface.co/datasets/PRIME-RL/Eurus-2-RL-Data}, which contains 25.3k problems that are mainly programming competition level.

\paragraph{Evaluation Benchmarks.}
For evaluation, we evaluate on the 1k held-out validation problems from the PRIME code dataset. For each prompt, we sample $k=8$ independent Python programs to compute the expected \texttt{Pass@1} (the average success rate across the 8 samples) and the \texttt{pass@8} metrics.

\paragraph{Reward Function.}
Following PRIME~\citep{cui2025prime}, we extract the generated Python program and evaluate it against a suite of test cases. The reward is defined as the fraction of tests passed:
\begin{equation}
r = \frac{\text{number of passed tests}}{\text{total number of tests}}.
\end{equation}
Compared to a strict binary reward setup, this formulation provides a denser learning signal, yielding values in $[0, 1]$ where $1$ indicates a fully correct solution and $0$ indicates complete failure.

\subsubsection{Geometry}

\paragraph{Training Dataset.}
We train on the 2.1k-problem training split of the Geometry3k dataset~\citep{lu2021inter, geometry3k_dataset}, available at \url{https://huggingface.co/datasets/hiyouga/geometry3k}.
Each problem in Geometry3k consists of a geometric diagram and an accompanying natural language question, often requiring multi-step spatial or logical reasoning.

\paragraph{Evaluation Benchmarks.}
We evaluate performance on the official 601-problem test split of Geometry3k. For each prompt, we generate 16 independent responses to calculate both the expected \texttt{Pass@1} (the average accuracy across 16 samples) and the \texttt{Pass@16} metrics.

\paragraph{Reward Function.}
Following verl~\citep{sheng2024verl}, we use the same reward function as in Countdown.

Appendix~\ref{appsec:dataexample} presents data examples from each of the training datasets.

\subsection{Models}
\label{appsec:used_llms}
We evaluate eight models spanning a diverse range of types, parameter scales, and model families. All models are sourced directly from their official Hugging Face repositories and used as released:
\begin{itemize}[leftmargin=10pt]
    \item Qwen3-1.7B-Base: \url{https://huggingface.co/Qwen/Qwen3-1.7B-Base};
    \item Qwen3-4B-Base: \url{https://huggingface.co/Qwen/Qwen3-4B-Base};
    \item Qwen3-8B-Base: \url{https://huggingface.co/Qwen/Qwen3-8B-Base};
    \item Qwen3-14B-Base: \url{https://huggingface.co/Qwen/Qwen3-14B-Base};
    \item Qwen2.5-VL-3B-Instruct: \url{https://huggingface.co/Qwen/Qwen2.5-VL-3B-Instruct};
    \item DeepSeek-R1-Distill-Qwen-1.5B: \url{https://huggingface.co/deepseek-ai/DeepSeek-R1-Distill-Qwen-1.5B};
    \item Llama-3.1-8B-Instruct: \url{https://huggingface.co/meta-llama/Llama-3.1-8B-Instruct};
    \item Mistral-7B-Instruct-v0.1: \url{https://huggingface.co/mistralai/Mistral-7B-Instruct-v0.1};
\end{itemize}

\begin{figure}[t]
    \centering
    \includegraphics[width=0.6\linewidth]{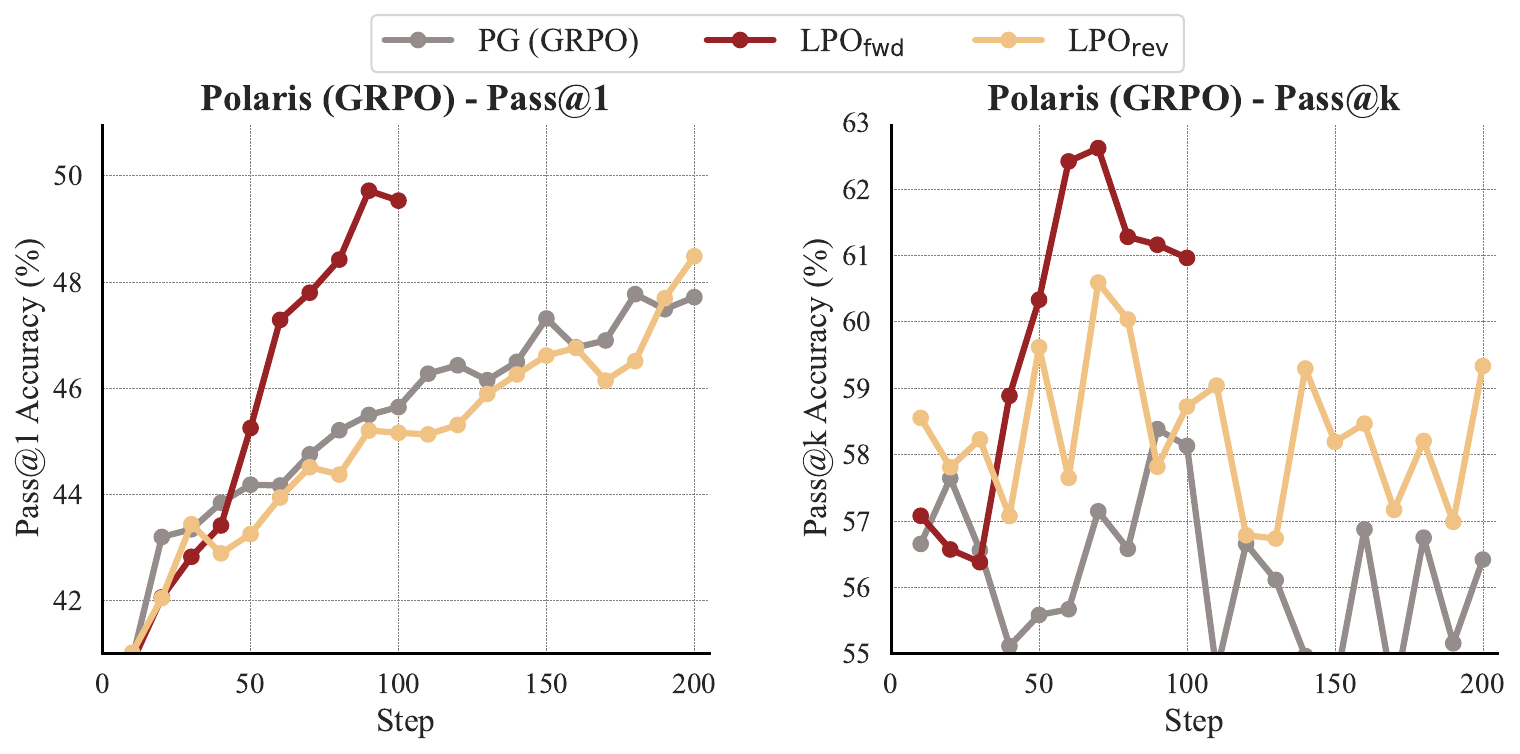}
    \vspace{-11pt}
    \caption{Scalability validation. We compare LPO with GRPO by training Qwen3-14B-Base on the larger Polaris dataset.}
    \vspace{-5pt}
    \label{fig:polaris}
\end{figure}
\subsection{Training Details}

We evaluate our method against three representative group-based policy gradient baselines: GRPO~\citep{shao2024grpo}, Dr.GRPO~\citep{liu2025drgrpo}, and MaxRL~\citep{tajwar2026maxrl}, all implemented within the verl framework~\citep{sheng2024verl}. 
Across all four reasoning scenarios, we sample a group of $K=8$ responses per prompt during training to estimate advantages or construct response simplex.
The generation temperature is set to $1.0$ with $\texttt{top\_p}=1.0, \texttt{top\_k}=-1.0$, 
and we disable the KL penalty by setting $\beta=0$, consistent with \citet{yu2025dapo}.
Evaluation generations use a lower temperature of $0.6$, following common practice~\citep{qu2025mopps}.

We tailor the batch sizes and context lengths according to the complexity of the specific benchmark. 
For the Math and PRIME-Code tasks, we set the training batch size to 256, the mini-batch size to 128, and the maximum response length to 4096 tokens. 
For the Countdown and Geometry tasks, we scale down the training batch size to 128 and the mini-batch size to 64, with the maximum response length capped at 1024 tokens.
This configuration performs two gradient updates per iteration, inherently introducing a mild off-policy drift. A strictly fully on-policy ablation is provided in Appendix~\ref{app:fully_onpolicy}.

Optimization is uniformly performed using Adam~\citep{kingma2014adam} with a learning rate of $1\mathrm{e}{-6}$ across all tasks.
The optimizer parameters are set to $\beta=(0.9,0.999)$ with a weight decay of $0.1$.
The clipping parameter is fixed at $\epsilon=0.2$. 
Given the highly non-linear parameter-space updates, we additionally apply token-level clipping~\citep{schulman2017ppo}.
The token-level log-density ratio $\delta_{k,i} = \log\pi_\theta(y_{k,i}|x,y_{k,<i}) - \log\pi_b(y_{k,i}|x,y_{k,<i})$ is clipped to $[\log(1{-}\epsilon), \log(1{+}\epsilon)]$ and then weighted by $c_k$ to form the final loss.

All experiments are conducted on 8 NVIDIA H20 GPUs.

\section{Extended Experimental Results}
\label{appsec:results}

\subsection{Scalability Validation}
\label{appsec:polaris}

To verify the scalability and extensibility of the LPO framework, we conduct additional experiments using the Qwen3-14B-Base model on the Polaris dataset, which contains approximately 53k complex reasoning problems.
We compare both LPO variants with the GRPO baseline.
As shown in Fig.~\ref{fig:polaris}, LPO-fwd exhibits remarkable sample efficiency, reaching the peak performance achieved by GRPO at 200 training steps within only 70 steps, while simultaneously providing significant improvements in both Pass@1 and Pass@k metrics.
For the LPO-rev variant, although its Pass@1 accuracy is comparable to GRPO, it shows superior robustness in maintaining Pass@k, effectively preserving response diversity.
These findings provide strong evidence that LPO is scalable and capable of maintaining its theoretical advantages alongside increases in model capacity and data volume.

\begin{figure}
    \centering
    \includegraphics[width=\linewidth]{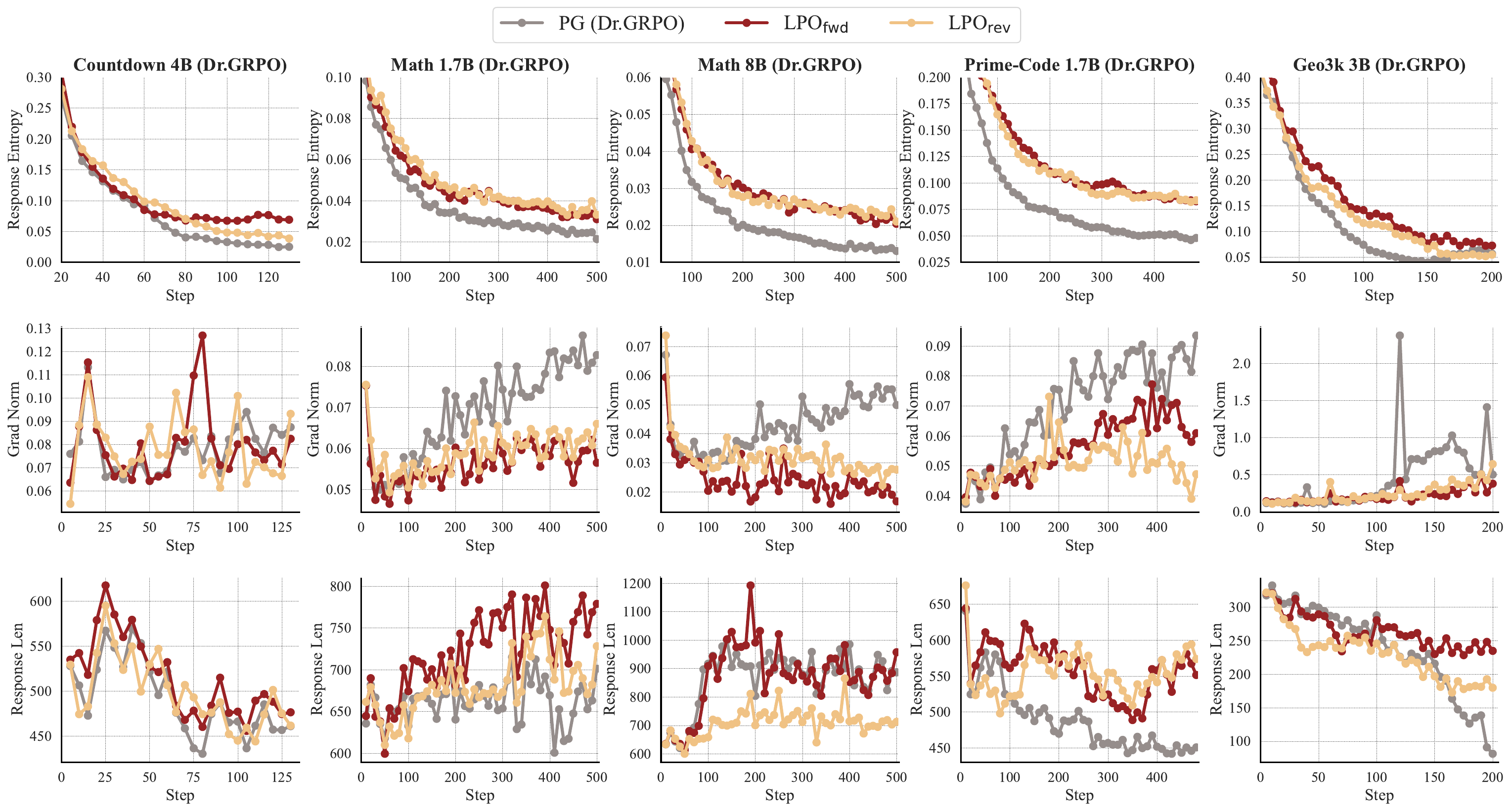}
    \vspace{-22pt}
    \caption{Training dynamics of LPO variants and Dr.GRPO. Rows from top to bottom respectively show the curves of response entropy, gradient norms, and response lengths.}
    \label{fig:drgrpodynamic}
\end{figure}

\begin{figure}
    \centering
    \includegraphics[width=\linewidth]{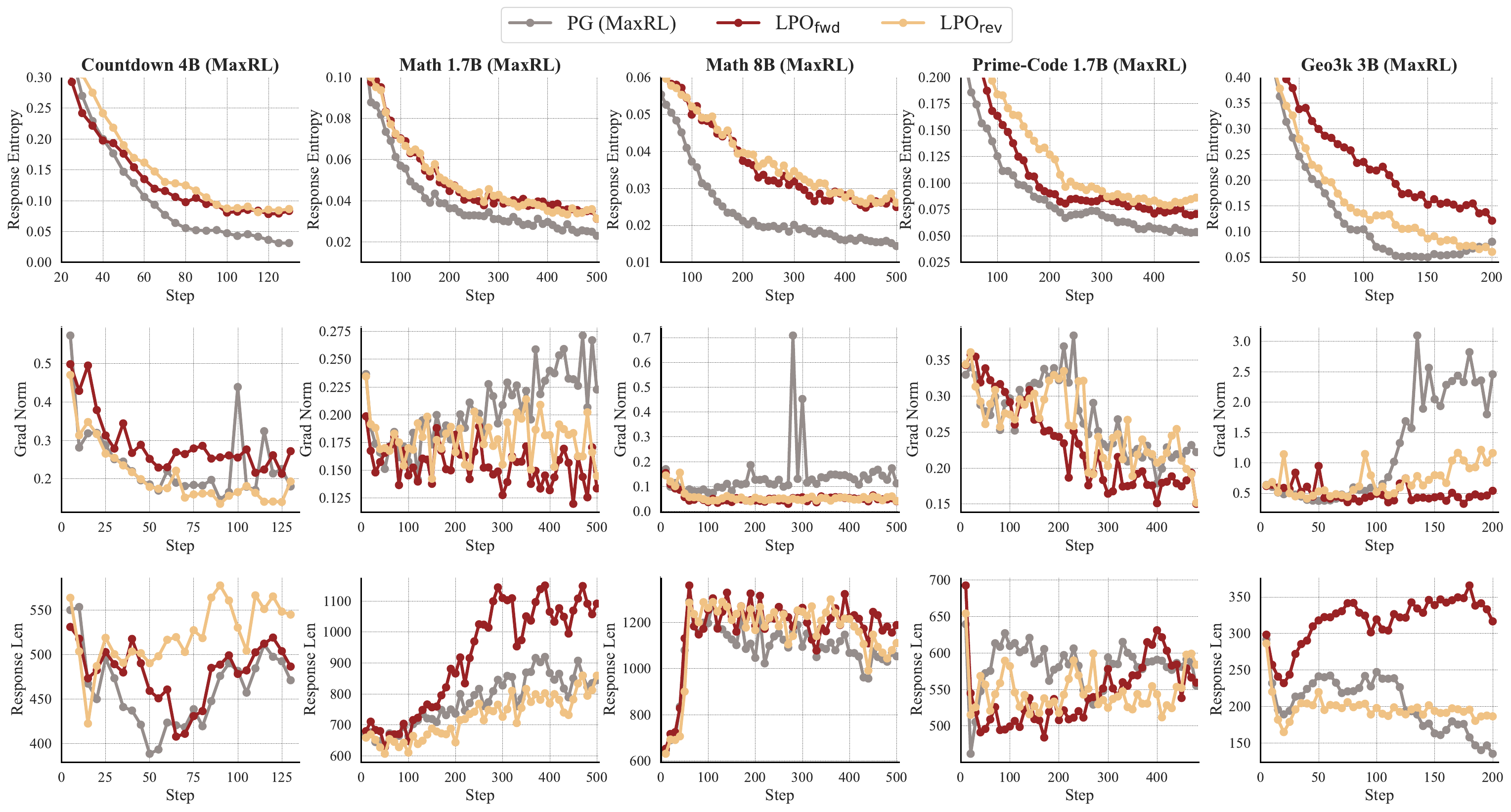}
    \vspace{-22pt}
    \caption{Training dynamics of LPO variants and MaxRL. Rows from top to bottom respectively show the curves of response entropy, gradient norms, and response lengths.}
    \label{fig:maxrldynamic}
\end{figure}

\subsection{Extended Training Dynamics}
\label{app:training_dynamics}

To corroborate the analysis presented in Sec.~\ref{sec:training_dynamics}, we provide the extended training dynamics of LPO compared against Dr.GRPO in Fig.~\ref{fig:drgrpodynamic} and MaxRL in Fig.~\ref{fig:maxrldynamic}. 

Consistent with the observations relative to the GRPO baseline in the main text, LPO variants demonstrate superior optimization properties across these baselines. Specifically, LPO maintains higher response entropy, exhibits lower and more stable gradient norms, and encourages longer response chains. These supplementary results further support the structural advantages of listwise projection.

\begin{figure}[t]
    \centering
    \includegraphics[width=\linewidth]{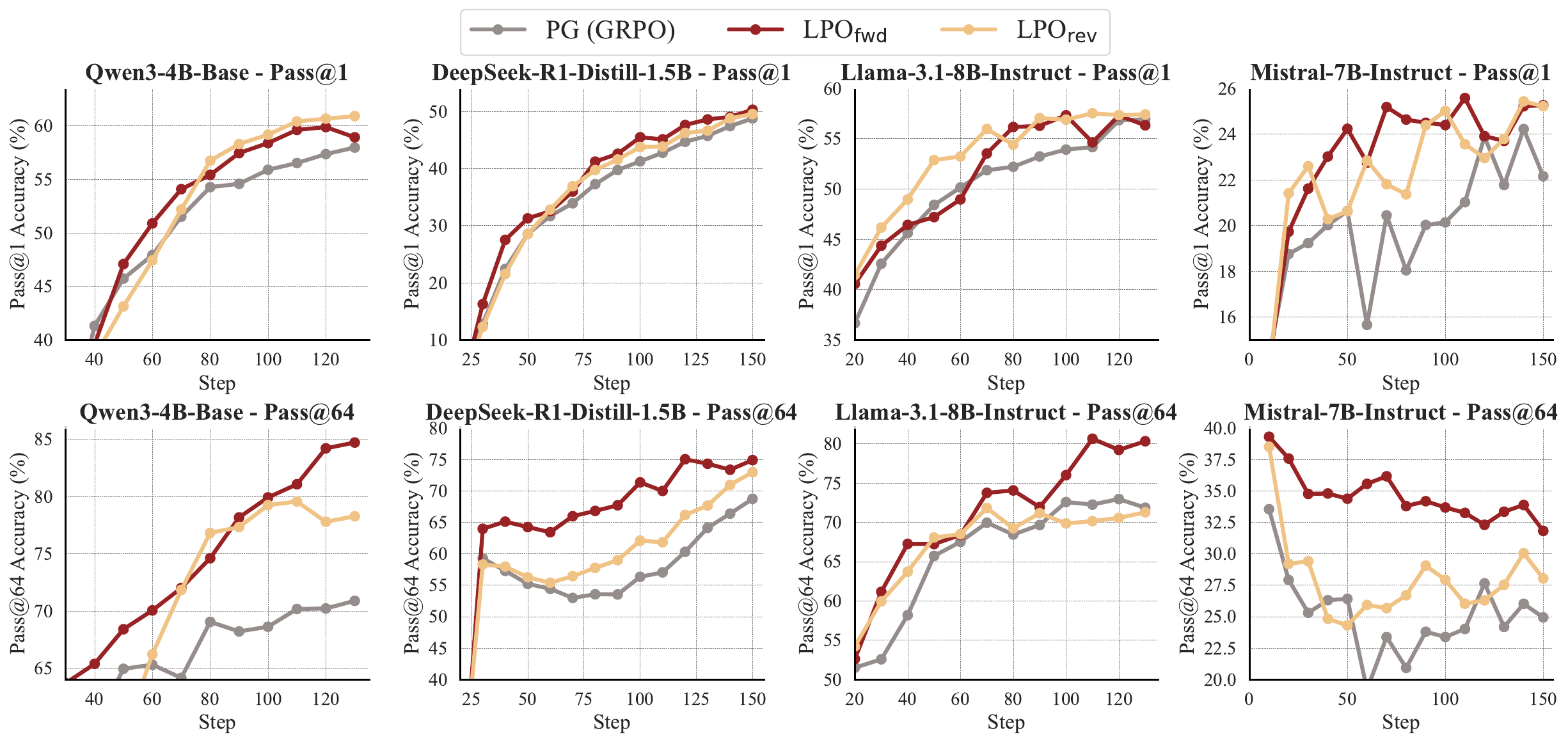}
    \vspace{-20pt}
    \caption{Generalization of LPO across diverse LLM families. Performance is evaluated on Countdown using Qwen, DeepSeek, Mistral, and Llama backbones.}
    \vspace{-5pt}
    \label{fig:families}
\end{figure}
\subsection{Generalization across LLM Families}
\label{appsec:llmfamily}

To evaluate the generalizability of LPO, we conduct experiments across four prominent LLM families: Qwen, DeepSeek, Mistral, and Llama.
These include models with different training paradigms such as base (only pre-trained), distilled, and instruction-tuned variants.
As shown in Fig.~\ref{fig:families}, 
across all evaluated LLMs, LPO consistently improves performance on the Countdown task over the PG baseline, with especially stable gains under Pass@64 evaluation.

\begin{figure}[t]
    \centering
    \includegraphics[width=0.6\linewidth]{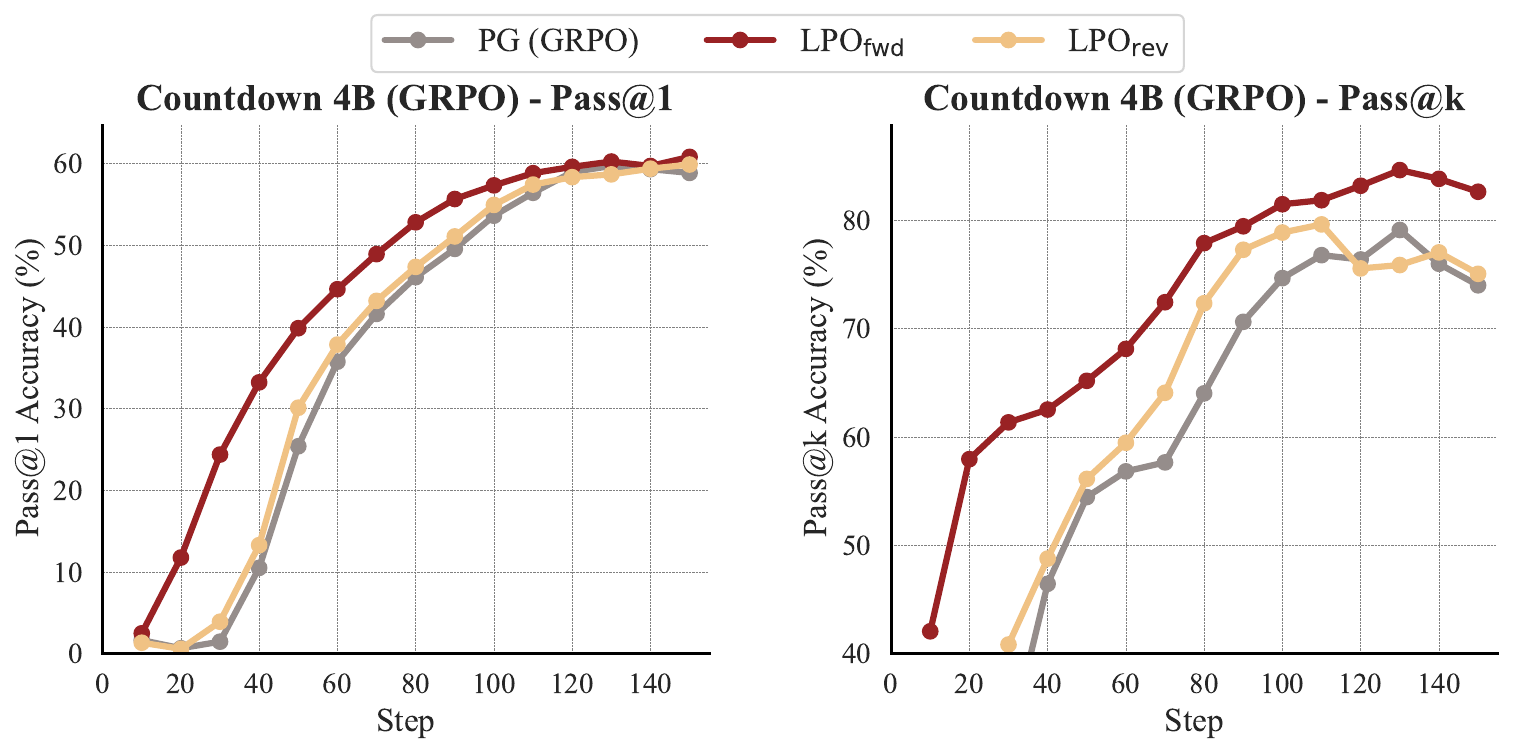}
    \vspace{-11pt}
    \caption{Empirical evaluation on the Countdown task under a fully on-policy regime (one gradient update per iteration).}
    \label{fig:fullonp}
\end{figure}
\subsection{Fully On-Policy Optimization}
\label{app:fully_onpolicy}

To empirically validate the theoretical connections established in Proposition~\ref{prop:pg_revkl}, we conduct an evaluation on Countdown under a strictly fully on-policy regime, as shown in Fig.~\ref{fig:fullonp}.
By setting both the batch size and the optimization mini-batch size to 256, we ensure exactly one gradient update is performed per iteration.
As predicted by our analysis, the training curves of $\text{LPO}_{\text{rev}}$ are practically indistinguishable from standard GRPO, supporting that the group-based PG objective mathematically collapses into the exact reverse KL projection at the on-policy point.
Furthermore, evaluating $\text{LPO}_{\text{fwd}}$ under this identical setup reveals its distinct exploration superiority: it demonstrates higher sample efficiency in early training stages and achieves a superior Pass@k accuracy.

\begin{table}[t]
  \centering
\renewcommand{\arraystretch}{1.3}
\caption{
Evaluation on mathematics benchmarks.
\textbf{Base} denotes the backbone without RLVR.
\textbf{Pass@1} and \textbf{Pass@k} are computed by averaging benchmark-level Avg@k and Pass@k scores across benchmarks, respectively.
\textbf{Bold} and \underline{underlined} values indicate the best and second-best results for each policy gradient baseline, respectively.
}
\vspace{-5pt}
  \resizebox{\linewidth}{!}{
    \begin{tabular}{crcccccccccc|cc}
    \toprule
    \multirow{2}[2]{*}{\textbf{Backbone}} & \multicolumn{1}{l}{\multirow{2}[2]{*}{\textbf{Method}}} & \textbf{MATH500} & \textbf{Olympiad.} & \multicolumn{2}{c}{\textbf{Minerva.}} & \multicolumn{2}{c}{\textbf{AMC23}} & \multicolumn{2}{c}{\textbf{AIME24}} & \multicolumn{2}{c|}{\textbf{AIME25}} & \multicolumn{1}{c}{\multirow{2}[2]{*}{\textbf{Pass@1$\uparrow$}}} & \multicolumn{1}{c}{\multirow{2}[2]{*}{\textbf{Pass@k$\uparrow$}}} \\
          &       & \textbf{Avg@1} & \textbf{Avg@1} & \textbf{Avg@4} & \textbf{pass@4} & \textbf{Avg@32} & \textbf{pass@32} & \textbf{Avg@32} & \textbf{pass@32} & \textbf{Avg@32} & \textbf{pass@32} &       &  \\
    \midrule
    \multirow{10}[8]{*}{\textbf{Qwen3-1.7B-Base}} & \multicolumn{1}{l}{\textbf{Base}} & 52.8  & 21.2  & 21.2  & 32.8  & 30.0    & 79.3  & 3.4   & 25.0    & 3.3   & 23.8  & 22.0  & 40.2  \\
\cmidrule{2-14}          & \multicolumn{1}{l}{\textbf{GRPO}} & 71.4  & 33.5  & 29.8  & 37.2  & 45.4  & 83.7  & 10.8  & 26.5  & 4.2   & 20.9  & 32.5  & 42.1  \\
          & \hspace{0.6em}$\hookrightarrow$$\boldsymbol{\mathrm{LPO_{fwd}}}$ & 72.0    & 38.1  & 29.9  & 37.5  & 50.1  & 83.4  & 13.2  & 33.7  & 8.6   & 29.6  & \textbf{35.3}  & \textbf{46.1}  \\
          & \hspace{0.6em}$\hookrightarrow$$\boldsymbol{\mathrm{LPO_{rev}}}$ & 73.0    & 37.1  & 29.2  & 36.4  & 47.0    & 83.1  & 13.9  & 26.6  & 9.6   & 22.9  & \underline{35.0}  & \underline{42.3}  \\
\cmidrule{2-14}          & \multicolumn{1}{l}{\textbf{DrGRPO}} & 69.2  & 36.1  & 29.8  & 36.8  & 43.3  & 76.2  & 8.5   & 25.4  & 6.3   & 30.4  & 32.2  & 42.2  \\
          & \hspace{0.6em}$\hookrightarrow$$\boldsymbol{\mathrm{LPO_{fwd}}}$ & 73.8  & 36.5  & 28.3  & 36.5  & 46.2  & 75.9  & 10.3  & 27.1  & 5.3   & 30.4  & \textbf{33.4}  & \underline{42.5}  \\
          & \hspace{0.6em}$\hookrightarrow$$\boldsymbol{\mathrm{LPO_{rev}}}$ & 70.0    & 36.7  & 28.6  & 38.1  & 45.6  & 78.9  & 10.3  & 31.7  & 6.3   & 26.7  & \underline{32.9}  & \textbf{43.9}  \\
\cmidrule{2-14}          & \multicolumn{1}{l}{\textbf{MaxRL}} & 72.6  & 35.3  & 28.6  & 36.9  & 42.4  & 79.0    & 10.6  & 30.8  & 4.8   & 24.6  & 32.4  & 42.8  \\
          & \hspace{0.6em}$\hookrightarrow$$\boldsymbol{\mathrm{LPO_{fwd}}}$ & 71.8  & 37.5  & 30.5  & 36.4  & 49.9  & 85.5  & 11.8  & 28.6  & 8.5   & 31.9  & \textbf{35.0}  & \textbf{45.6}  \\
          & \hspace{0.6em}$\hookrightarrow$$\boldsymbol{\mathrm{LPO_{rev}}}$ & 72.6  & 35.8  & 28.8  & 36.3  & 46.1  & 82.6  & 10.7  & 28.3  & 7.7   & 32.6  & \underline{33.6}  & \underline{45.0}  \\
    \midrule
    \multirow{10}[8]{*}{\textbf{Qwen3-8B-Base}} & \multicolumn{1}{l}{\textbf{Base}} & 68.0    & 33.7  & 31.7  & 44.1  & 46.5  & 84.3  & 12.1  & 39.9  & 7.9   & 31.8  & 33.3  & 50.0  \\
\cmidrule{2-14}          & \multicolumn{1}{l}{\textbf{GRPO}} & 86.2  & 51.9  & 40.4  & 46.1  & 63.8  & 79.1  & 24.0    & 52.1  & 19.5  & 40.7  & 47.6  & 54.5  \\
          & \hspace{0.6em}$\hookrightarrow$$\boldsymbol{\mathrm{LPO_{fwd}}}$ & 86.4  & 55.8  & 42.3  & 48.3  & 69.1  & 95.1  & 29.3  & 51.0    & 19.1  & 38.7  & \textbf{50.3}  & \textbf{58.3}  \\
          & \hspace{0.6em}$\hookrightarrow$$\boldsymbol{\mathrm{LPO_{rev}}}$ & 85.0    & 53.9  & 41.1  & 46.9  & 67.0    & 93.1  & 23.3  & 45.7  & 21.6  & 40.2  & \underline{48.7}  & \underline{56.5}  \\
\cmidrule{2-14}          & \multicolumn{1}{l}{\textbf{DrGRPO}} & 85.8  & 54.7  & 42.2  & 48.4  & 67.7  & 89.7  & 24.9  & 56.3  & 19.3  & 47.0   & \underline{49.1}  & \textbf{60.4}  \\
          & \hspace{0.6em}$\hookrightarrow$$\boldsymbol{\mathrm{LPO_{fwd}}}$ & 87.4  & 51.6  & 42.6  & 48.3  & 70.2  & 91.5  & 25.6  & 59.5  & 19.8  & 38.4  & \textbf{49.5}  & \underline{59.4}  \\
          & \hspace{0.6em}$\hookrightarrow$$\boldsymbol{\mathrm{LPO_{rev}}}$ & 84.6  & 51.0    & 42.0    & 47.8  & 64.9  & 91.4  & 26.0    & 53.0    & 17.9  & 35.3  & 47.7  & 56.9 \\
\cmidrule{2-14}          & \multicolumn{1}{l}{\textbf{MaxRL}} & 86.4  & 53.6  & 42.6  & 48.9  & 66.0    & 93.4  & 23.9  & 48.6  & 18.9  & 41.7  & 48.6  & 58.2  \\
          & \hspace{0.6em}$\hookrightarrow$$\boldsymbol{\mathrm{LPO_{fwd}}}$ & 89.4  & 54.5  & 44.8  & 52.3  & 69.0    & 94.5  & 23.9  & 57.6  & 21.3  & 47.8  & \underline{50.5}  & \textbf{63.1}  \\
          & \hspace{0.6em}$\hookrightarrow$$\boldsymbol{\mathrm{LPO_{rev}}}$ & 87.6  & 55.8  & 45.3  & 52.3  & 70.1  & 92.5  & 22.5  & 52.6  & 22.5  & 43.6  & \textbf{50.6}  & \underline{60.3}  \\
    \bottomrule
    \end{tabular}%
  \label{tab:matheval}%
  }
  \vspace{-10pt}
\end{table}%

\subsection{Evaluation Results}
Table~\ref{tab:matheval} presents the final evaluation on mathematics benchmarks, with $k$ configurations following standard practices~\citep{gao2025prompt, qu2025mopps}. 
Furthermore, to assess out-of-distribution~(OOD) generalization, Table~\ref{tab:ood} compares LPO against counterpart PG baselines (all trained on MATH using Qwen3-8B-Base) across general reasoning tasks: MMLU-Pro~\cite{wang2024mmlu}, ARC-c~\cite{clark2018think}, and GPQA-diamond~\cite{rein2024gpqa}.
While specific LPO variants can improve the overall average, OOD evaluation exhibits inherent variance, suggesting multi-domain joint training as a natural direction for future work.

\begin{table}[htbp]
  \centering
  \caption{Out-of-Distribution evaluation of LPO against baseline methods trained on MATH.}
  \vspace{-5pt}
  \label{tab:ood}
  \resizebox{0.5\linewidth}{!}{
  \begin{tabular}{lcccc}
    \toprule
    \multicolumn{1}{l}{\multirow{2}[2]{*}{\textbf{Method}}} & \textbf{ARC-c} & \textbf{MMLU-Pro} & \textbf{GPQA} & \multicolumn{1}{l}{\multirow{2}[2]{*}{\textbf{Avg.$\uparrow$}}} \\
    &       \textbf{Avg@32} & \textbf{Avg@32} & \textbf{Avg@32} &  \\
    \midrule
    GRPO & 33.4 & 56.0 & 25.3 & 38.2 \\
    \hspace{0.6em}$\hookrightarrow$LPO & 38.4 & 53.5 & 23.7 & \textbf{38.5} \\
    \midrule
    Dr.GRPO & 33.2 & 55.4 & 23.8 & 37.5 \\
    \hspace{0.6em}$\hookrightarrow$LPO & 36.4 & 58.5 & 25.6 & \textbf{40.2} \\
    \midrule
    MaxRL & 22.5 & 49.8 & 18.7 & 30.3 \\
    \hspace{0.6em}$\hookrightarrow$LPO & 26.3 & 51.2 & 19.3 & \textbf{32.3} \\
    \bottomrule
  \end{tabular}
  }
\end{table}

\section{Data Examples}
\label{appsec:dataexample}
The prompt templates for MATH and Geometry3k are adopted from the official \texttt{verl} framework; the template for Countdown follows the format introduced in \cite{tinyzero}, and we directly use the prompts for PRIME code from \cite{cui2025prime}.

\begin{tcolorbox}[example, title=MATH example]
\textbf{Prompt:}

The points $(9, -5)$ and $(-3, -1)$ are the endpoints of a diameter of a circle. What is the sum of the coordinates of the center of the circle? Let's think step by step and output the final answer within \texttt{\textbackslash boxed\{\}}.

\textbf{Ground-Truth Answer:}

0
\end{tcolorbox}
\begin{tcolorbox}[example, title=Countdown example]
\textbf{Prompt:}

A conversation between User and Assistant. The user asks a question, and the Assistant solves it. The assistant first thinks about the reasoning process in the mind and then provides the user with the answer.

User: Using the numbers [2, 54, 17], create an equation that equals 35. You can use basic arithmetic operations (+, -, *, /) and each number can only be used once. Show your work in $<$think$>\ </$think$>$ tags. And return the final answer in $<$answer$>\ </$answer$>$ tags, for example $<$answer$> (1 + 2) / 3 <$/answer$>$.

Assistant: Let me solve this step by step.

$<$think$>$

\end{tcolorbox}

\begin{tcolorbox}[example, title=PRIME-Code example]
\textbf{System Prompt:}

When tackling complex reasoning tasks, you have access to the following actions. Use them as needed to progress through your thought process.

[ASSESS] 
[ADVANCE]
[VERIFY]
[SIMPLIFY]
[SYNTHESIZE]
[PIVOT]
[OUTPUT]

You should strictly follow the format below:

[ACTION NAME]

\# Your action step 1

\# Your action step 2

\# Your action step 3

...

Next action: [NEXT ACTION NAME]

\textbf{Prompt}

Problem:
Given a natural number N less than or equal to 12, find the smallest natural number such that the number of divisors is exactly N.

Constraints

* $1\le N \le 12$

Input:
One natural number N is given in one line.

Output:
Output the smallest natural number on a line so that the number of divisors is exactly N.

Examples

Input
1
Output
1
Input
2
Output
2
Input
3
Output
4

Write Python code to solve the problem. Present the code in 

\texttt{python}

\begin{verbatim}
Your code
\end{verbatim}

at the end.

\end{tcolorbox}

\begin{tcolorbox}[example, title=Geometry3k example]
\textbf{Prompt:}

\includegraphics[width=0.22\linewidth]{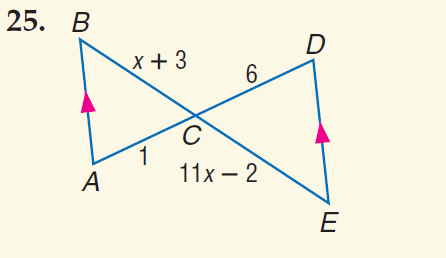}
Find x. You FIRST think about the reasoning process as an internal monologue and then provide the final answer. The reasoning process MUST BE enclosed within $<$think$>\ </$think$>$ tags. The final answer MUST BE put in \texttt{\textbackslash boxed\{\}}.

\textbf{Ground-Truth Answer:}

4
\end{tcolorbox}